%% file: main.tex
\documentclass[twoside]{article}
\usepackage[accepted]{aistats2025}

\usepackage{graphicx}      % include this line if your document contains figures
\usepackage[round]{natbib}

\usepackage{amsmath,amssymb}
\usepackage[utf8]{inputenc} % allow utf-8 input
\usepackage[T1]{fontenc}    % use 8-bit T1 fonts
\usepackage{hyperref}       % hyperlinks
\usepackage{url}            % simple URL typesetting
\usepackage{booktabs}       % professional-quality tables
\usepackage{amsfonts}       % blackboard math symbols
\usepackage{nicefrac}       % compact symbols for 1/2, etc.
\usepackage{microtype}      % microtypography
\usepackage{xcolor}         % colors
\usepackage{multirow}
\usepackage{multicol}
\usepackage{resizegather}

\usepackage{amsthm}
\usepackage[mathscr]{eucal}

\newtheorem{Assumption}{Assumption}
\usepackage{color}
\usepackage{todonotes}
\usepackage{bm}
\usepackage{soul}
\usepackage{pgfgantt}
\usepackage{graphicx}
\usepackage{xcolor}
\usepackage{datetime}
\usepackage{standalone}
\usepackage{subcaption}
\usepackage{hyperref}
\newtheorem{thm}{\bf Theorem}
 
\newtheorem{lem}[thm]{\bf Lemma}
\newtheorem{cor}[thm]{\bf Corollary}
\usepackage{algorithmic}
\usepackage{algorithm}

\newcommand{\Hk}{{\mathcal{H}_k}}

\newcommand{\bs}[1]{{\boldsymbol{#1}}}

\makeatletter
\newcommand{\markthis}[3]{% #1 = marker, #2 = label, #3 = relation
  \overset{% the marker and the label
    \textup{\makebox[0pt]{#1}}%
    \def\@currentlabel{#1}%
    \ltx@label{#2}%
  }{% the relation
    #3%
  }%
}
\makeatother

\newcommand{\Pstar}{\mathcal{P}^*}
\newcommand{\Phat}{\hat{\mathcal{P}}}
\newcommand{\BPhat}{\mathcal{B}^{\varepsilon_t}(\Phat_t)}

\newcommand{\infQBPhat}{\inf_{\mathcal{Q}\in\mathcal{B}^{\varepsilon_t}(\Phat_t)}}
\newcommand{\EPstar}{\mathbb{E}_{c\sim\Pstar}}
\newcommand{\EPhat}{\mathbb{E}_{c\sim\Phat}}
\newcommand{\EPstart}{\mathbb{E}_{c\sim\Pstar_t}}
\newcommand{\EPhatt}{\mathbb{E}_{c\sim\Phat_t}}
\newcommand{\EQ}{\mathbb{E}_{c\sim\mathcal{Q}}}
\newcommand{\UCB}{\text{UCB}_t}
\newcommand{\LCB}{\text{LCB}_t}
\newcommand{\xt}{x_t}
\newcommand{\fxt}{f(\xt, c)}
\newcommand{\xstar}{x^*}
\newcommand{\fxtstar}{f(\xstar_t, c)}

\newcommand{\diamC}{{\mathrm{D}_\mathcal{C}}}
\newcommand{\diamX}{{\mathrm{D}_\mathcal{X}}}

\begin{document}
% \maketitle
\twocolumn[
\aistatstitle{Wasserstein Distributionally Robust Bayesian Optimization with~Continuous~Context}
\aistatsauthor{ Francesco Micheli \And Efe C. Balta \And  Anastasios Tsiamis \And John Lygeros }
\aistatsaddress{ ETH Z\"urich \And inspire AG \& ETH Z\"urich \And ETH Z\"urich \And ETH Z\"urich } ]

%%%%%%%%%%%%%%%%%%%%%%%%%%%%%%%%%%%%%%%%%%%%%%%%%%%%%%%%%%%%%%%%%%%%%%%%%%%%%%%%
\begin{abstract}
We address the challenge of sequential data-driven decision-making under context distributional uncertainty. This problem arises in numerous real-world scenarios where the learner optimizes black-box objective functions in the presence of uncontrollable contextual variables. We consider the setting where the context distribution is uncertain but known to lie within an ambiguity set defined as a ball in the Wasserstein distance. We propose a novel algorithm for Wasserstein Distributionally Robust Bayesian Optimization that can handle continuous context distributions while maintaining computational tractability. Our theoretical analysis combines recent results in self-normalized concentration in Hilbert spaces and finite-sample bounds for distributionally robust optimization to establish sublinear regret bounds that match state-of-the-art results. Through extensive comparisons with existing approaches on both synthetic and real-world problems, we demonstrate the simplicity, effectiveness, and practical applicability of our proposed method.
\end{abstract}

%%%%%%%%%%%%%%%%%%%%%%%%%%%%%%%%%%%%%%%%%%%%%%%%%%%%%%%%%%%%%%%%%%%%%

\section{INTRODUCTION}
Bayesian Optimization (BO) has emerged as a powerful algorithm for zero-order optimization of expensive-to-evaluate black-box functions, with applications ranging from hyperparameters tuning to scientific discovery and robotics~\citep{ueno2016combo,li2019accelerating,ru2020bayesian,shahriari2015taking}. In the standard BO setting, the learner sequentially selects points to evaluate the unknown objective function and uses the observed data to update a surrogate model that captures the function's behavior. 
In the contextual BO setting, the objective function depends on an additional variable, called the context, which cannot be controlled by the learner~\citep{krause2011contextual,valko2013finite,kirschner2019stochastic}. Typically, the context distribution is used to model the uncertainty of the learner related to uncontrollable environmental variables.
When the distribution of the context variable is known, the BO algorithm can be used to solve the Stochastic Optimization (SO) problem, where the objective is to maximize the reward of the unknown function in expectation with respect to the context distribution
$$
\max_{x \in \mathcal{X}} \mathbb{E}_{c \sim \mathcal{P}} [f(x, c)] \ .
$$
However, in many real-world scenarios, the learner does not have access to the true context distribution, but only to an approximate one. This can happen, e.g., when the context distribution is estimated from historical data, and only a finite number of samples are available. This results in a distributional mismatch between the distribution available to the learner for optimization and the true distribution of the context variable. 
To formally account for the effect of the distributional mismatch, Distributionally Robust Optimization (DRO) has recently gained considerable attention, especially in the sampled data settings~\citep{rahimian2019distributionally,kuhn2019wasserstein,gao2024wasserstein}.
In DRO, the learner optimizes the reward under the worst-case distribution of the context within a so-called ambiguity set $\mathcal{B}$ that captures the uncertainty of the learner about the true context distribution
\begin{equation}
\max_{x \in \mathcal{X}} \inf_{\mathcal{Q} \in \mathcal{B}} \mathbb{E}_{c \sim \mathcal{Q}} [f(x, c)] \ .
\end{equation}
The advantage of the robust approach is that, by appropriately choosing the ambiguity set $\mathcal{B}$, we can guarantee that the reward computed for the DRO problem lower-bounds the reward for the true unknown context distribution. 

In this work, we introduce Wasserstein Distributionally Robust Bayesian Optimization (WDRBO), a novel algorithm that combines the principles of BO and DRO to address the challenge of sequential data-driven decision-making under context distributional uncertainty.
We consider ambiguity sets defined as balls in the Wasserstein distance~\citep{kuhn2019wasserstein} which allows for a flexible and intuitive way to model the uncertainty in the context distribution. We design a computationally tractable algorithm and analyze its performance in two settings: the \textbf{General WDRBO} setting, where at each time-step the Wasserstein ambiguity set is provided to the learner, and the \textbf{Data-Driven WDRBO} setting, in which we assume that the true context distribution is time-invariant and the Wasserstein ambiguity set is built using the past context observations.

Our main contributions are as follows:
\begin{itemize}
    \item We propose a novel, computationally tractable algorithm for Wasserstein Distributionally Robust Bayesian Optimization that handles continuous context distributions. Our approach exploits an approximate reformulation based on Lipschitz bounds of the acquisition function, circumventing the need for context discretization.
    
    \item We establish a cumulative expected regret bound of order $\tilde{O}(\sqrt{T}\gamma_T)$ for the general WDRBO setting, where $T$ is the number of iterations and $\gamma_T$ is the maximum information gain. For the data-driven setting, we obtain sublinear regret guarantees without requiring assumptions on the rate of decay of the ambiguity set radius.
    
    \item We derive novel Lipschitz bounds for the mean and variance estimates, and leverage recent finite-sample bounds for Wasserstein DRO to address the dimensionality challenges in continuous context spaces.
    
    \item We provide comprehensive empirical evaluations on synthetic and real-world problems, demonstrating that our method achieves competitive performance with significantly lower computational complexity compared to existing DRBO approaches.
\end{itemize}

The rest of the paper is organized as follows. In Section~\ref{sec:related_work}, we review related work. In Section~\ref{sec:problem_formulation}, we introduce the problem formulation. In Section~\ref{sec:wdrbo}, we present the proposed algorithm and provide the theoretical analysis. In Section~\ref{sec:experiments}, we present the experimental results. Finally, in Section~\ref{sec:conclusions}, we conclude the paper and discuss future work.

%%%%%%%%%%%%%%%%%%%%%%%%%%%%%%%%%%%%%%%%%%%%%%%%%%%%%%%%%%%%%%%%%%%%%%%%%%%%%%%

\section{RELATED WORK}
\label{sec:related_work}
The foundation of DRBO was laid by~\cite{kirschner2020distributionally}, who introduced the concept of distributional robustness in BO. They propose a BO formulation that is robust to the worst-case context distribution within an ambiguity set defined by the Maximum Mean Discrepancy (MMD) distance. While groundbreaking, the inner worst-case calculation requires at each iteration the solution of a convex optimization problem that renders this approach computationally viable only when the context space is discrete and with low cardinality.
A quadrature-based scheme for DRBO is proposed in~\cite{nguyen2020distributionally}, but their algorithm is limited to the simulator setting where at each iteration the learner is allowed to choose the context.
\cite{husain2024distributionally} develops a DRBO formulation for $\ phi$-divergence-based ambiguity sets, but their formulation has some implicit requirements on the support of the distributions captured by the ambiguity set.
Recognizing these computational limitations,~\cite{tay2022efficient} proposed a set of approximate techniques using worst-case sensitivity analysis based on Taylor's expansions. These methods offer better computational complexity for multiple descriptions of ambiguity sets at the expense of performance and regret bounds that scale linearly with the worst-case sensitivity approximation error.
To avoid the challenges of context space discretization,~\cite{huang2024stochastic} proposes a kernel density estimation step that uses the available context samples to estimate a continuous context distribution. The estimated context distribution is then sampled and the samples are used in a DRBO formulation where the $\phi$-divergence ambiguity sets capture the distributional uncertainty introduced by the density estimation step.

The regret analysis of the existing literature on DRBO builds on the GP-UCB formulation of~\cite{srinivas2009gaussian,srinivas2012information}, we instead exploit self-normalizing concentration bounds in Reproducing Kernel Hilbert Space (RKHS)~\citep{abbasi2013online,kirschner2020distributionally,whitehouse2023improved}. We address a gap in the DRBO literature and analyze the continuous context distribution setting under the Wasserstein-based ambiguity set. We leverage recent advancements in Wasserstein DRO literature~\citep{gao2023finite,gao2024wasserstein} to provide state-of-the-art regret rates in the data-driven setting.

%%%%%%%%%%%%%%%%%%%%%%%%%%%%%%%%%%%%%%%%%%%%%%%%%%%%%%%%%%%%%%%%%%%%%%%%%%%%%%%

\section{PROBLEM FORMULATION}
\label{sec:problem_formulation}
We consider an unknown objective function $f: \mathcal{X} \times \mathcal{C} \rightarrow \mathbb{R}$, where $\mathcal{X} \subset \mathbb{R}^{d_x}$ is the input space and $\mathcal{C} \subset \mathbb{R}^{d_c}$ is the context space. The learner's goal is to maximize the expected value of the function under the context distribution by sequentially selecting points to evaluate and receiving noisy observations of the function. 
More specifically, at each iteration $t = 1, 2, \ldots$, the learner selects a point $x_t \in \mathcal{X}$ to query the function, and observes the context $c_t \in \mathcal{C}$ and a noisy output $y_t = f(x_t, c_t) + \eta_t$. The context sample $c_t$ is assumed to be an independent sample from some unknown, time-dependent, context distribution $\mathcal{P}^*_t$, while $\eta_t$ is a zero-mean $R$-sub-Gaussian noise, where an upper bound on $R$ is known.

\subsection{Wasserstein Distributionally Robust Objective}
In this work, we consider the setting of distributionally robust optimization, where the learner does not have access to the true context distribution $\mathcal{P}^*_t$, but instead optimizes for the expected reward under the worst-case distribution within an ambiguity set
\begin{equation}\label{eq:dro}
\max_{x \in \mathcal{X}} \inf_{\mathcal{Q} \in \mathcal{B}^{\varepsilon_t}(\hat{\mathcal{P}}_t)} \mathbb{E}_{c \sim \mathcal{Q}} [f(x, c)] \ .
\end{equation}

The time-dependent ambiguity set $\mathcal{B}^{\varepsilon_t}(\hat{\mathcal{P}}_t)$ is defined as a ball in the Wasserstein distance centered at the distribution $\hat{\mathcal{P}}_t$ and with radius $\varepsilon_t$~\citep{kuhn2019wasserstein}. This is the set of all distributions that are within a Wasserstein distance $\varepsilon_t$ from the center distribution $\Phat_t$
\begin{equation*}
    \mathcal{B}^{\varepsilon_t}(\Phat_t) = \left\{ \mathcal{Q} \in \mathcal{P}(\mathcal{C}) : d_W(\mathcal{Q}, \Phat_t) \leq \varepsilon_t \right\}.
\end{equation*}
The type-1 Wasserstein metric $d_{\mathrm{W}} : \mathcal{M}(\mathbb{Q}) \times \mathcal{M}(\mathbb{Q}) \rightarrow \mathbb{R}_{\geq 0}$ defines the distance between two distributions $\mathcal{Q}_1$ and $\mathcal{Q}_2$ as
\begin{equation*}\label{eq:Wasserstein_distance}
        d_{W}\left(\mathcal{Q}_{1}, \mathcal{Q}_{2}\right):=\inf_{\Pi} \left\{\int_{\mathbb{Q}\times\mathbb{Q}}\left\|\bs{q}_{1}-\bs{q}_{2}\right\| \Pi\left(\mathrm{d} \bs{q}_{1}, \mathrm{d} \bs{q}_{2}\right)\right\},
\end{equation*}
where the transportation map $\Pi$ takes values in the set of joint distributions of $\bs{q}_1$ and $\bs{q}_2$ with marginals $\mathcal{Q}_1$ and $\mathcal{Q}_2$, and $\| \cdot \|$ is the euclidean norm.

\subsection{Regularity Assumptions and Surrogate Model}
The BO algorithm maintains a surrogate model of the objective function, which is used to guide the selection of the next query point. We use a regularized least squares estimator of the function $f$ in the RKHS\citep{abbasi2013online, kirschner2018information} under the assumption that $f$ is an unknown fixed member of the RKHS $\mathcal{H}_k$ that is specified by the positive semi-definite kernel $k: \mathcal{Z}\times\mathcal{Z}\rightarrow\mathbb{R}$, where $\mathcal{Z} = \mathcal{X} \times \mathcal{C}$. Here we define $z=[x^\top,c^\top]^\top$ to keep the notation compact. We assume that the spaces $\mathcal{X}$ and $\mathcal{C}$ are compact. We define the norm of a function $g\in \Hk$ as $\|g\|_\Hk = \sqrt{g^\top g} = \sqrt{\langle g, g \rangle_\Hk}$. We also assume that the unknown function $f$ has bounded RKHS norm, i.e., $\|f\|_\Hk \leq B$, for some $B > 0$, and that the kernel $k$ is bounded, i.e., $k(z, z') \leq 1$, for all $z,\ z' \in \mathcal{Z}$. The assumptions made here are common in the BO literature, we point the reader to e.g.~\cite{bogunovic2021misspecified} for the analysis of bandits optimization with misspecified RKHS. The details of the following derivations are available in the Appendix Section~\ref{sec:rkhs}.

Given the dataset $\mathcal{D}_t = \{(z_i, y_i)\}_{i=1}^t$, and regularization parameter $\lambda>0$, the regularized least-squares regression problem in RKHS is written as follows:
\begin{equation}
    \min_{\mu \in \Hk} \sum_{i=1}^t \left(y_i - \mu(z_i)\right)^2 + \lambda \|\mu\|_\Hk^2.
\end{equation}
The resulting least squares estimator is
\begin{equation}\label{eq:mu}
    \mu_t(z) = k_t(z)^\top (K_t + \lambda I)^{-1} y_{1:t}\ ,
\end{equation}
where $k_t(z) = [k(z, z_1), \ldots, k(z, z_t)]^\top$, $K_t = [k(z_i, z_j)]_{i,j=1}^t$, and $y_{1:t} = [y_1, \ldots, y_t]^\top$. 
We also define
\begin{equation}\label{eq:sigma}
    \sigma_t^2(z) = \frac{1}{\lambda}\left(k(z, z) - k_t(z)^\top (K_t + \lambda I)^{-1} k_t(z)\right)\ .
\end{equation}
Under suitable assumptions on the prior and the noise distribution, $\mu_t$ and $\sigma_t^2$ correspond to the posterior mean and variance of a Gaussian process with kernel $k$ conditioned on the observations $y_{1:t}$~\citep{scholkopf2002learning,williams2006gaussian}.

We state here a fundamental result adapted from~\cite{abbasi2013online} that provides probabilistic finite-sample confidence guarantees for the least squares estimator~\eqref{eq:mu}.
\begin{lem}
    \label{lem:rkhs_UCB_LCB}
   [\citep[Th. 3.11]{abbasi2013online}]
    Let $\mathcal{Z}\subset\mathbb{R}^d$, where $d=d_x + d_c$, and $f:\mathcal{Z}:\rightarrow \mathbb{R}$ be a member of $\Hk$, with $\| f \|_\Hk \leq B$, and let $\eta_t$ be $\mathcal{F}_t$ measurable and $R$-sub-Gaussian conditionally on $\mathcal{F}_t$. Then, for any $\lambda>0$, with probability $1-\delta$, we have that for all $z\in\mathcal{Z}$ and all $t\geq 1$: 
    \begin{equation}
        |\mu_{t-1}(z) - f(z)| \leq \beta_t \sigma_{t-1}(z)\ ,
    \end{equation}
    with \begin{equation}\label{eq:beta_definition}
        \beta_t := R \sqrt{2\log\left( \frac{\det(I + \lambda^{-1}K_{t-1})^\frac{1}{2} }{\delta}\right)} + \lambda^{\frac{1}{2}}B \ ,
    \end{equation}
    where $\mu_{t-1}$ and $\sigma_{t-1}$ are defined as in equations~\eqref{eq:mu},\eqref{eq:sigma}.
\end{lem}

We also introduce here the \textit{maximum information gain}~\citep{srinivas2009gaussian, chowdhury2017kernelized, vakili2021information}, a fundamental kernel-dependent quantity that quantifies the complexity of learning in RKHS
$$\gamma_t := \sup_{z_1,z_2,\dots,z_t} \log \det \left( I + \lambda^{-1}K_{t-1} \right)\ .$$

In order to derive the main results in the following sections, we require the kernel to satisfy the following Lipschitz property
\begin{Assumption}[Lipschitz property]\label{ass:Lip}
There exists a $L>0$ such that for any $z,z'\in \mathcal{X}\times\mathcal{C}$, $d(z,z'):= \|k(\cdot,z)-k(\cdot,z')\|_\Hk\le L\|z-z'\|$ .
\end{Assumption}
As we prove in Lemma~\ref{app_lem:kernels_that_satisfy_ass_1} in the Appendix, Assumption~\ref{ass:Lip} is verified for popular kernels, e.g. the squared exponential kernel, and some kernels in the Mat\'ern family satisfy.

\section{WASSERSTEIN DISTRIBUTIONALLY ROBUST BAYESIAN OPTIMIZATION}
\label{sec:wdrbo}
In classical BO, following the rich literature of \textbf{optimism in the face of uncertainty}, the learner selects the query point $x_t$ by maximizing the Upper Confidence Bound (UCB) function~\citep{auer2002finite}. This provides a trade-off between exploration and exploitation, and results in provable regret guarantees~\citep{srinivas2009gaussian}.

Departing from the classical approach, and inspired by~\cite{kirschner2020distributionally}, we adopt a robust approach, 
where we consider the optimization of a robustified version of the UCB function
\begin{equation}\label{eq:intractable_acquisition}
\xt = \arg \max_{x\in\mathcal{X}} \infQBPhat \EQ \left[ UCB_t(x,c) \right] \ ,
\end{equation}
where
\begin{equation}\label{eq:UCB_definition}
UCB_t(x,c) = \mu_{t-1}(x,c) + \beta_t \sigma_{t-1}(x,c),
\end{equation}
with $\mu_{t}$, $\sigma_t$, and $\beta_t$ are given in~\eqref{eq:mu},~\eqref{eq:sigma},~\eqref{eq:beta_definition} respectively.

Similar to~\cite{kirschner2020distributionally}, we will analyze two settings, which differ in the way the ambiguity set is obtained. We first consider the \textbf{General WDRBO} setting, where at each time-step the Wasserstein ambiguity set is provided to the learner, and then turn to the \textbf{Data-Driven WDRBO} setting, in which the Wasserstein ambiguity set is built using the past context observations under the assumption that the true context distribution is time-invariant.

All proofs along with supporting derivations and lemmas are provided in Section~\ref{app_sec:main_proofs} of the Appendix.

To evaluate the performance of the proposed algorithm we look at the notion of \textbf{regret}. Regret is used to capture the difference in performance between some algorithm and a benchmark algorithm that has access to privileged information. The definition of regret and the choice of benchmark is not unique, and the one chosen here differs from the ones used in the DRBO literature ~\cite{kirschner2019stochastic},~\cite{husain2024distributionally},~\cite{tay2022efficient}.

We will consider the following definitions of \textbf{instantaneous expected regret}:
\begin{equation}
    r_t = \EPstart{[\fxtstar]} - \EPstart{[\fxt]} \ ,
\end{equation}
 and \textbf{cumulative expected regret}:
\begin{equation}\label{eq:ECRR}
    R_T = \sum_{t=1}^{T} r_t\ .
\end{equation}
The benchmark solution $\xstar_t$ is the optimal solution to the true stochastic optimization problem at time-step $t$, given access to the true function $f$ and context distribution $\Pstar_t$, i.e.,
\begin{equation*}
\xstar_t = \arg \max_{x\in\mathcal{X}} \EPstart{[f(x, c)]} \ .
\end{equation*}
Hence, this definition of regret captures the (cumulative) sub-optimality gap, between some proposed algorithm and the optimal solution to the true stochastic optimization problem.

\subsection{General WDRBO}
In the \textbf{General WDRBO} setting, at each time-step $t$, the center $\Phat_t$ and the radius $\varepsilon_t$ of the Wasserstein ambiguity set are provided to the learner. This represents the setting where there is some understanding of what the context distribution is, e.g. with weather or prices forecast, but there is still some uncertainty about its distribution.

To make the robust problem~\ref{eq:intractable_acquisition} tractable, we introduce a well-known result from the Wasserstein DR optimization literature~\citep{kuhn2019wasserstein,gao2024wasserstein} that has been adapted to our problem. 
\begin{lem}
    \label{cor:wdro}
    Let $f: \mathcal{X} \times \mathcal{C} \rightarrow \mathbb{R}$ be a function that is $L^f_c(x)$-Lipschitz in the context space, i.e. $|f(x, c) - f(x, c')| \leq L^f_c(x) \|c - c'\|$, for all $c, c' \in \mathcal{C}$. Let $\mathcal{B}^{\varepsilon}(\hat{\mathcal{P}})$ be a Wasserstein ambiguity set defined as a ball of radius $\varepsilon$ in the Wasserstein distance centered at the distribution $\hat{\mathcal{P}}$. Then, for any $x\in\mathcal{X}$ and for any distribution $\tilde{\mathcal{P}} \in \mathcal{B}^{\varepsilon}(\hat{\mathcal{P}})$, we have that
    \begin{equation}
        | \mathbb{E}_{c \sim \tilde{\mathcal{P}}}[f(x, c)] - \mathbb{E}_{c \sim \hat{\mathcal{P}}}[f(x, c)] | \leq \varepsilon L^f_c(x)\ .
    \end{equation}
\end{lem}

Lemma~\ref{cor:wdro} provides a simple Lipschitz-based bound on the worst-case expectation for any distribution in the ambiguity set. By combining Lemma~\ref{cor:wdro} with Assumption~\ref{ass:Lip} we obtain a tractable approximation of the robust maximization problem~\ref{eq:intractable_acquisition}. Thus, at each time-step $t$ the query point $x_t$ is selected by the following \emph{acquisition function}
\begin{equation}\label{eq:WUCB}
\xt = \arg \max_{x\in\mathcal{X}} \EPhatt \left[ \UCB(x,c) \right] - \varepsilon_t L^{\UCB}(x) \, ,
\end{equation}
where $L^{\UCB}(x)$ is the Lipschitz constant of the function $\UCB$ with respect to the context variable $c$, evaluated at $x$. 
The resulting algorithm for WDRBO is provided in Algorithm~\ref{alg:BO_algo}.
\begin{algorithm}\caption{General Algorithm for WDRBO}\label{alg:BO_algo}
\begin{algorithmic}
\FOR{$t = 1$ to $T$}
    \STATE $\mu_{t-1},\ \sigma_{t-1} \gets \text{fit}(\mathcal{D}_{t-1})$
    \STATE $x_t = \arg\max_{x \in \mathcal{X}} \EPhat[\UCB(x, c)] - \varepsilon_t L^{\UCB}(x)$ 
    \STATE The environment returns $c_t$ and $y_t = f(x_t,c_t) + \eta_t\ $, where $c_t \sim \mathcal{P}^*_t \in \mathcal{B}^{\varepsilon_t}(\hat{\mathcal{P}}_t)$ %and $\eta_t \sim \mathcal{N}(0,\sigma^2)$ 
    \STATE $\mathcal{D}_{t} \gets \mathcal{D}_{t-1} \cup \{(x_t,c_t,y_t)\}$
\ENDFOR
\end{algorithmic}
\end{algorithm}

We remark that, unlike the algorithm proposed in~\cite{kirschner2019stochastic},~\cite{husain2024distributionally},~\cite{tay2022efficient}, that rely on discrete context distributions, neither Algorithm~\ref{alg:BO_algo} nor the following theoretical analysis require a discrete context space or that the distributions in the ambiguity have finite support. The only practical limitation imposed by our algorithm when the center is a continuous context distribution, is the ability to perform the numerical integration required to compute the expectation.

While in Algorithm~\ref{alg:BO_algo} the Lipschitz constant $L^{\UCB}(x)$ can be computed at each timestep $t$ from the fitted UCB function $\UCB$, in the following lemma we derive a novel upper bound on $L^{\UCB}(x)$ that will be useful for the theoretical analysis of the regret.
\begin{lem}\label{lem:B_bar}
Let $0< \delta < 1$ be a failure probability and let 
    $\bar{B}_t := \lambda^{-\frac{1}{2}}R \sqrt{2\log\left(\frac{\det(I + \lambda^{-1}K_{t-1})^\frac{1}{2}}{\delta}\right)} + B \leq \lambda^{-\frac{1}{2}} \left( \left( R \sqrt{2 \log \frac{1}{\delta} } + R \sqrt{2\gamma_t} \right) + B \right)$ .
    
    Then, with probability $1-\delta$ for all $t\geq 1$ we have 
    $\|\mu_{t-1}\|_{\mathcal{H}_k} \leq \bar{B}_t.$
    Further, if Assumption~\ref{ass:Lip} holds we have:
    
    (i) With probability $1-\delta$, for any $z, z' \in \mathcal{X} \times \mathcal{C}$:
    $|\mu_{t-1}(z) - \mu_{t-1}(z')| \leq \bar{B}_t L \|z-z'\|.$
    
    (ii) For any $z, z' \in \mathcal{X} \times \mathcal{C}$:
    $|\beta_t\sigma_{t-1}(z) - \beta_t\sigma_{t-1}(z')| \leq \beta_t \lambda^{-\frac{1}{2}} L \|z-z'\| = \bar{B}_t L \|z-z'\|.$
    
    Therefore, with probability $1-\delta$, the UCB function is Lipschitz continuous with constant:
    $$L^{\UCB} \leq 2\bar{B}_t L.$$
\end{lem}

We can now turn to the derivation of the bound on the instantaneous expected regret for the General WDRBO setting.
\begin{thm}[Instantaneous expected regret]\label{thm:inst_regret}
    Let Assumption~\ref{ass:Lip} hold. Fix a failure probability $0<\delta<1$.
    With probability at least $1-\delta$, for all $t\ge 1$ the instantaneous expected regret can be bounded by
    \begin{equation}
        r_t \leq \EPstart \left[ 2 \beta_t \sigma_{t-1}(\xt, c) \right] + 2 \varepsilon_t L^{\UCB}(\xstar_t)
    \end{equation}    
\end{thm}
We can observe that the first term has the same expression as the instantaneous regret of the GP-UCB~\cite{srinivas2009gaussian}, while the second term captures the effect of the distributional uncertainty which depends on the maximum distribution shift as specified by $\varepsilon_t$, and on a sensitivity term that is bounded by the Lipschitz constants $L^{\UCB}(x_t)$ computed at the selected input $x_t$.

\begin{thm}[Cumulative expected regret]\label{thm:cum_regret}
    Let Assumption~\ref{ass:Lip} hold and let $L^{\UCB}(x)$ be a Lipschitz constant with respect to the context $c$ for $\UCB(x,c)$. Fix a failure probability $0<\delta<1$.
    With probability at least $1-2\delta$, the cumulative expected regret after $T$ steps can be bounded as:
    \begin{equation}
        R_T \leq 4 \beta_T \sqrt{ T \gamma_T + 4 \log (\frac{6}{\delta})} + \sum_{t=1}^{T} \varepsilon_t 2 L^{\UCB}(\xstar_t)\ ,
    \end{equation}
    where $\gamma_{T}$ is the maximum information gain at time $T$.
\end{thm}
Note that the cumulative expected regret is a random quantity, as the expectation is taken only with respect to the contexts.
We can combine Lemma~\ref{lem:B_bar} with Theorem~\ref{thm:cum_regret} to derive the regret rate for the cumulative expected regret.
\begin{cor}[General WDRBO Regret Order]\label{cor:wdrbo_regret}
    Let $0< \delta < 1$ be a failure probability and let Assumption~\ref{ass:Lip} hold. Then, with probability $1-2\delta$, the cumulative expected regret is of the order of
    \[ R_T = \tilde{O}\left(\sqrt{T}\gamma_T+ \sqrt{\gamma_T}\sum_{t=1}^{T} \varepsilon_t\right) \ .
\]
For the Squared Exponential kernel, this reduces to
    \[ R_T = \tilde{O}\left(\sqrt{T} + \sum_{t=1}^{T} \varepsilon_t\right) \ ,
\]
where $\tilde{O}$ omits logarithmic terms. 
\end{cor}

The second term depends on the sum of all radii $\sum_{t=1}^T \varepsilon_t$. Hence, a sufficient condition in order to get sublinear regret guarantees, is that the radii converge to $0$ sufficiently fast. If, e.g., $\varepsilon_t = O(t^{-\frac{1}{2}})$, we obtain $R_T = \tilde{O}\left( \sqrt{T}\gamma_T \right)$. This can also occur in certain situations like the data-driven setting that we analyze next.

\subsection{Data-Driven WDRBO}
In the \textbf{Data-Driven WDRBO} we still rely on Algorithm~\ref{alg:BO_algo}, but differently from the general setting we need to build the Wasserstein ambiguity set using the past context observations. With the assumption that the unknown context distribution is time-invariant, i.e. $\Pstar_t:=\Pstar$, $t\in 1,\dots,T$, we build the ambiguity set center $\Phat_t$ as the empirical distribution of the past observed contexts, i.e.
$$\Phat_t = \frac{1}{t} \sum_{i=1}^{t} \mathcal{I}_{\{c=c_i\}}\ ,$$
where $\mathcal{I}_{\{c=c^i\}}$ is the indicator function centered on the context sample $c^i$, and we derive a bound on the sequence of radii $\varepsilon_t$ using finite-sample concentration results.

Using finite-sample results for the convergence of empirical measures in Wasserstein distance~\cite{fournier2015rate, fournier2022convergence} it is possible to bound the size of $\varepsilon_t$ such that with high probability the true context distribution $\Pstar$ is contained in the ambiguity set $\BPhat$. Unfortunately, this approach suffers from the so-called curse of dimensionality with respect to the dimension $d_c$ of the context. To circumvent this issue we propose a novel result that leverages recent finite-sample concentration results from~\cite{gao2023finite}. Instead of focusing on the rate of convergence of the empirical distribution $\Phat_t$ to the true unknown $\Pstar$, we focus on the rate at which the worst-case expected cost concentrates around the expected cost under the true context distribution.

To extend the result of~\cite{gao2023finite} to the WDRBO setting we require an additional covering argument over all possible UCB functions, which depend on all the possible sampling histories of the algorithm. Define the class of functions 
\begin{align*}
\mathcal{U}(m,b,s)=&\left\{h: h(z) = \mu(z) + \beta \sigma(z)\ ,\right. \\ 
&\left. \|\mu\|_\Hk\leq m,\ \beta\leq b,\ \sigma(z) \in \bs{\sigma}_s\right\}\ ,
\end{align*}
where 
$$\bs{\sigma}_s = \left\{\sigma: 
\sigma(z) = \|\mathcal{M} k(\cdot, z)\|_\Hk,\ \| \mathcal{M} \|_{op}\leq s^{-\frac{1}{2}} \right\}\ .$$

Let $\mathcal{N}_{\infty}(\rho,\mathcal{U}(m,b,s))$ be its covering number under the infinity norm, up to precision $\rho$. Let $\diamX$, $\diamC$ denote the diameters of the sets $\mathcal{X},\mathcal{C}$ respectively.

\begin{lem}\label{cor:data_driven_UCB_main}
Let $0<\delta<1$ be a failure probability, and let $\UCB(x,c) = \mu_{t-1}(x,c) + \beta_t \sigma_{t-1}(x,c) \in \mathcal{U}(\bar{B}_t,\bar{B}_t,\lambda)$. Then, with probability at least $1-2\delta$, for any $x\in\mathcal{X}$ we have:
\begin{align*}
    &\left| \mathbb{E}_{c \sim \Pstar}[\UCB(x, c)] - \mathbb{E}_{c \sim \Phat_t}[\UCB(x, c)] \right| \leq \\
    &\qquad \qquad \qquad \leq \varepsilon_t L^\UCB(x) + \rho_t,
\end{align*}
where 
\begin{equation}
    \varepsilon_t\!=\!\sqrt{2\diamC^2\frac{\log 1/\delta\!+\! d_x\log(1+2t\diamX)+\log \mathcal{N}_{\infty}(t^{-1},\mathcal{U}(\bar{B}_t,\bar{B}_t,\lambda))}{t}},
\end{equation}
and \[\rho_t = (1+2L\bar{B}_t)t^{-1},\]
with $\bar{B}_t$ as defined in Lemma~\ref{lem:B_bar}.
\end{lem}

We can now use the bound of Lemma~\ref{cor:data_driven_UCB_main} to derive the data-driven analogous of Theorem~\ref{thm:inst_regret} and Theorem~\ref{thm:cum_regret}. Note that since the term $\rho_t$ is not dependent on the decision variable at timestep $t$, the robustified acquisition function in Algorithm~\ref{alg:BO_algo} remains unchanged.

\begin{thm}[Data-driven instantaneous expected regret]\label{thm:inst_regret_data}
    Let Assumption~\ref{ass:Lip} hold. Fix a failure probability $0<\delta<1$.
    With probability at least $1-2\delta$, for all $t\ge 1$ the instantaneous expected regret for the data-driven setting can be bounded by
    \begin{equation}
        r_t \leq \EPstar \left[ 2 \beta_t \sigma_{t-1}(\xt, c) \right] + 2 \varepsilon_t L^{\UCB}(\xstar_t) + 2\rho_t
    \end{equation}
    where $\varepsilon_t$ and $\rho_t$ are as defined in Lemma~\ref{cor:data_driven_UCB_main}.
\end{thm}

\begin{thm}[Data-driven cumulative expected regret]\label{thm:cum_regret_data}
    Let Assumption~\ref{ass:Lip} hold and let $L^{\UCB}(x)$ be a Lipschitz constant with respect to the context $c$ for $\UCB(x,c)$. Fix a failure probability $0<\delta<1$.
    With probability at least $1-3\delta$, the cumulative expected regret for the data-driven setting can be bounded as:
    \begin{equation}
        R_T \leq 4 \beta_T \sqrt{T(\gamma_T + 4 \log(6/\delta))} + \sum_{t=1}^{T} \left(2\varepsilon_t L^{\UCB}(\xstar_t) + 2\rho_t\right)\ .
    \end{equation}
    where $\varepsilon_t$ and $\rho_t$ are as defined in Lemma~\ref{cor:data_driven_UCB_main}, and $\gamma_{T}$ is the maximum information gain at time $T$.
\end{thm}

For the Squared Exponential kernel, we have that $\gamma_t = O(\log^{d+1}(t))$ from~\cite{vakili2021information}, and $\log \mathcal{N}_{\infty}(t^{-1},\mathcal{U}(\bar{B}_t,\bar{B}_t,\lambda)) \leq \kappa \left(1+\log(t\bar{B}_t)\right)^{1+d} + \kappa \left(1+\log t\right)^{1+2d}$
from Lemma~\ref{app_lem:covering_square_exponential} which follows from~\cite{yang2020function}. With these, we can obtain the regret order for data-driven WDRBO Regret Order with Squared Exponential Kernel.

\begin{cor}[Data-driven WDRBO Regret Order for Squared Exponential Kernel]\label{cor:datadriven}
    Let $0< \delta < 1$ be a failure probability and let Assumption~\ref{ass:Lip} hold. For the Squared Exponential kernel, with probability $1-3\delta$, the cumulative expected regret in the data-driven setting is bounded by:
    \begin{equation}
        R_T = \tilde{O}\left(\sqrt{T}\right)
    \end{equation}
    where $\tilde{O}$ omits logarithmic terms.
\end{cor}

The proposed analysis is not limited to the squared exponential kernel and following the same procedure, using the bounds on the covering number $\mathcal{N}_\infty$ derived in~\cite{yang2020function}, it is possible to derive similar bounds for other commonly used kernels experiencing either exponential or polynomial eigendecay.

The rate derived in Corollary~\ref{cor:datadriven} shows that a sublinear regret is achievable when the dependency on the covering number and the maximum information gain are well behaved. These are linked to the smoothness of the kernel.
It is not yet clear whether this is a fundamental limitation or if it is an artifact of our proving technique. We will leave this for future work.

Note that with the proposed data-driven WDRBO, we have a principled way to choose the sequence of radii $\varepsilon_t$ that provides a probabilistic guarantee on the maximum distance between the expectation under the true context generating distribution $\Pstar$ and the expectation under the empirical distribution $\Phat_t$. This makes Algorithm~\ref{alg:BO_algo} a practical tool to handle continuous context distributions, in contrast with~\cite{kirschner2020distributionally},~\cite{husain2024distributionally},~\cite{tay2022efficient}, where it is assumed that the true distribution is supported on a finite number of contexts. The proposed approach differs also from~\cite{huang2024stochastic} where their DRO-KDE algorithm robustifies against the gap between the approximate context distribution obtained by KDE from the observed contexts and the empirical distribution obtained by sampling it.

\section{EXPERIMENTS}
\label{sec:experiments}
In this section, we analyze the performance of the proposed algorithm and compare it with the algorithms in the literature. We will start with a simple example that showcases the effect of the robust acquisition function~\eqref{eq:WUCB} in the general setting. We will then provide an extensive comparison of the algorithms in the data-driven setting, as we consider it the most relevant and more challenging in practice. 

To highlight the need for robustness against context distribution shifts, we consider the \textbf{general DRBO} setting with fixed context distributions $\Phat_t=\mathcal{N}(0.5,0.1)$ for all $t = 1,\dots,100$ and $\Pstar_t=\mathcal{N}(0.6,0.2)$ for all $t = 1,\dots,100$, and the unknown function
$$ f(x,c) =  1 - \frac{|c-0.5|}{|x|+0.2} - \sqrt{|x|+0.05} \ .$$
A plot of the function and its optima under $\Phat_t$ and $\Pstar_t$ is shown in Fig.~\ref{fig:maxima}. We compare the performance of the proposed WDRBO as in Algorithm~\ref{alg:BO_algo} with $\varepsilon_t = 0.1$ for all $t = 1,\dots,100$, and ERBO (Empirical Risk BO), the non-robust variant of WDRBO that assumes $\varepsilon_t = 0$ for all $t = 1,\dots,100$.  

In Fig.~\ref{fig:regret_general} we show that the robust WDRBO results in a lower cumulative regret than ERBO. This is because ERBO solves the stochastic optimization problem assuming that the context is distributed according to the ambiguity set center $\Phat_t$, while WDRBO optimizes for the worst-case distribution in the ambiguity set of radius $0.1$.
In this simple setting, since the radius $\varepsilon_t$ remains constant over time, following the result of Theorem~\ref{thm:cum_regret}, the cumulative expected regret shows a linear trend.

\begin{figure}
\begin{subfigure}{.705\columnwidth}
    \centering
    \includegraphics[width=1\textwidth]{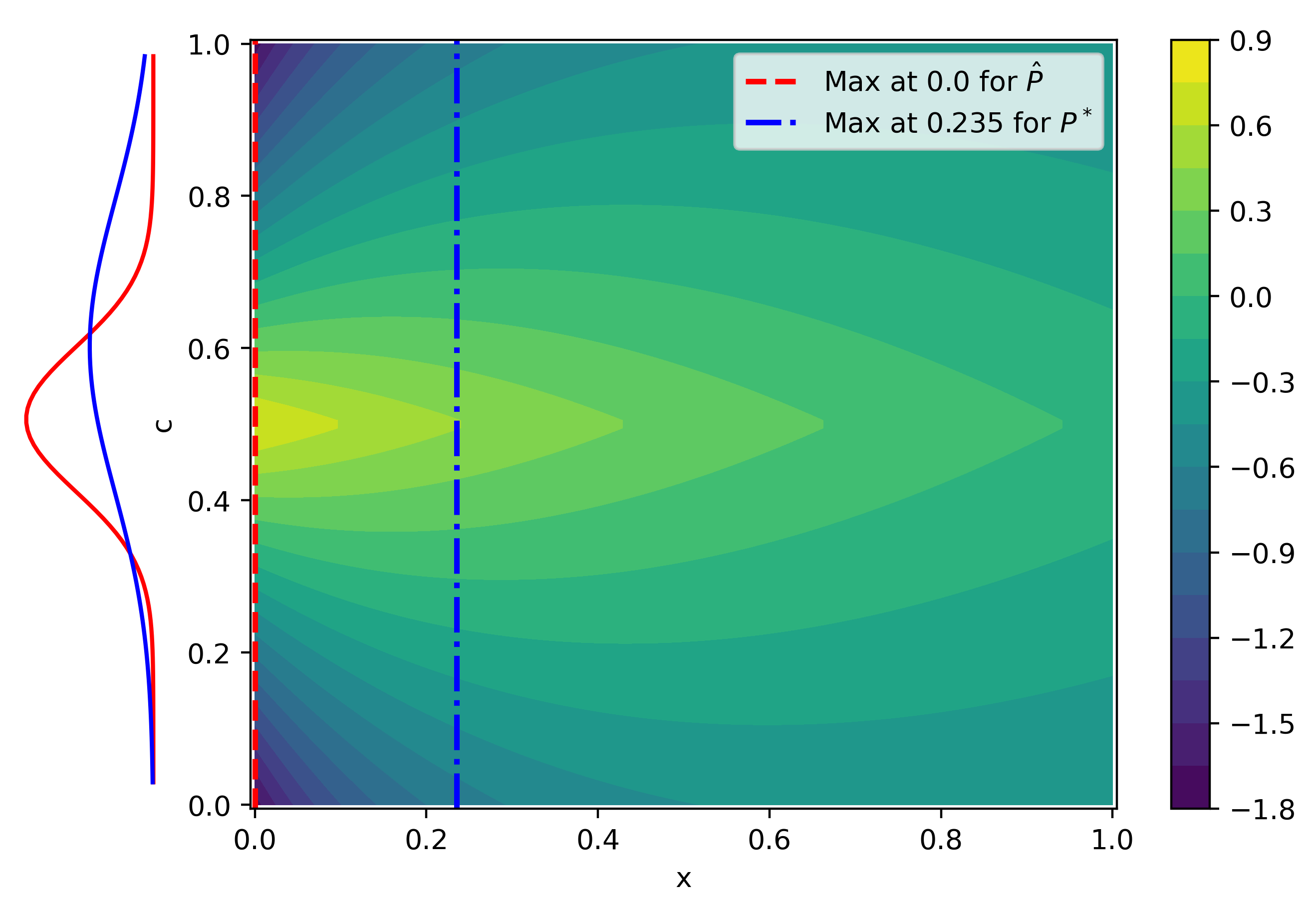}
    \caption{}\label{fig:maxima}
\end{subfigure}%
\begin{subfigure}{.286\columnwidth}
    \centering
    \includegraphics[width=1\textwidth]{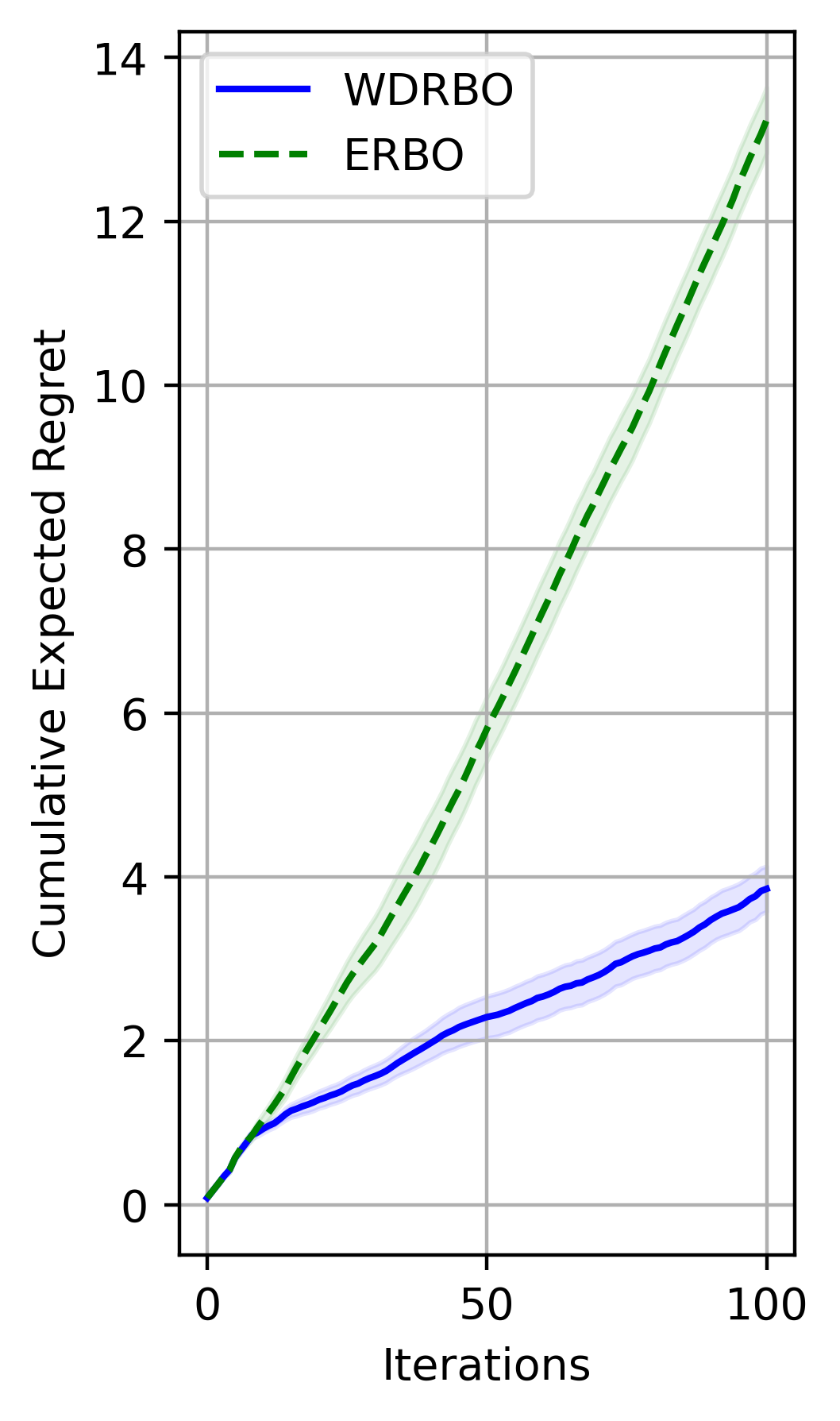}
    \caption{}\label{fig:regret_general}
\end{subfigure}
\caption{(a) Function $f(x,c)$ and its optima under optima under $\Phat_t$ and $\Pstar_t$. (b) Mean and standard error of the cumulative expected regret.}
\end{figure}

For the \textbf{data-driven DRBO} setting we adopt the setup of~\cite{huang2024stochastic} and provide a comparison of the different methods on synthetic function and the realistic problems.\footnote{The code is available at the following link~\href{https://github.com/frmicheli/WDRBO}{https://github.com/frmicheli/WDRBO} .} We will compare the algorithms' performance based on the cumulative expected regret as in~\eqref{eq:ECRR}.
\color{black}
\begin{figure*}[t]
  \includegraphics[width=\textwidth]{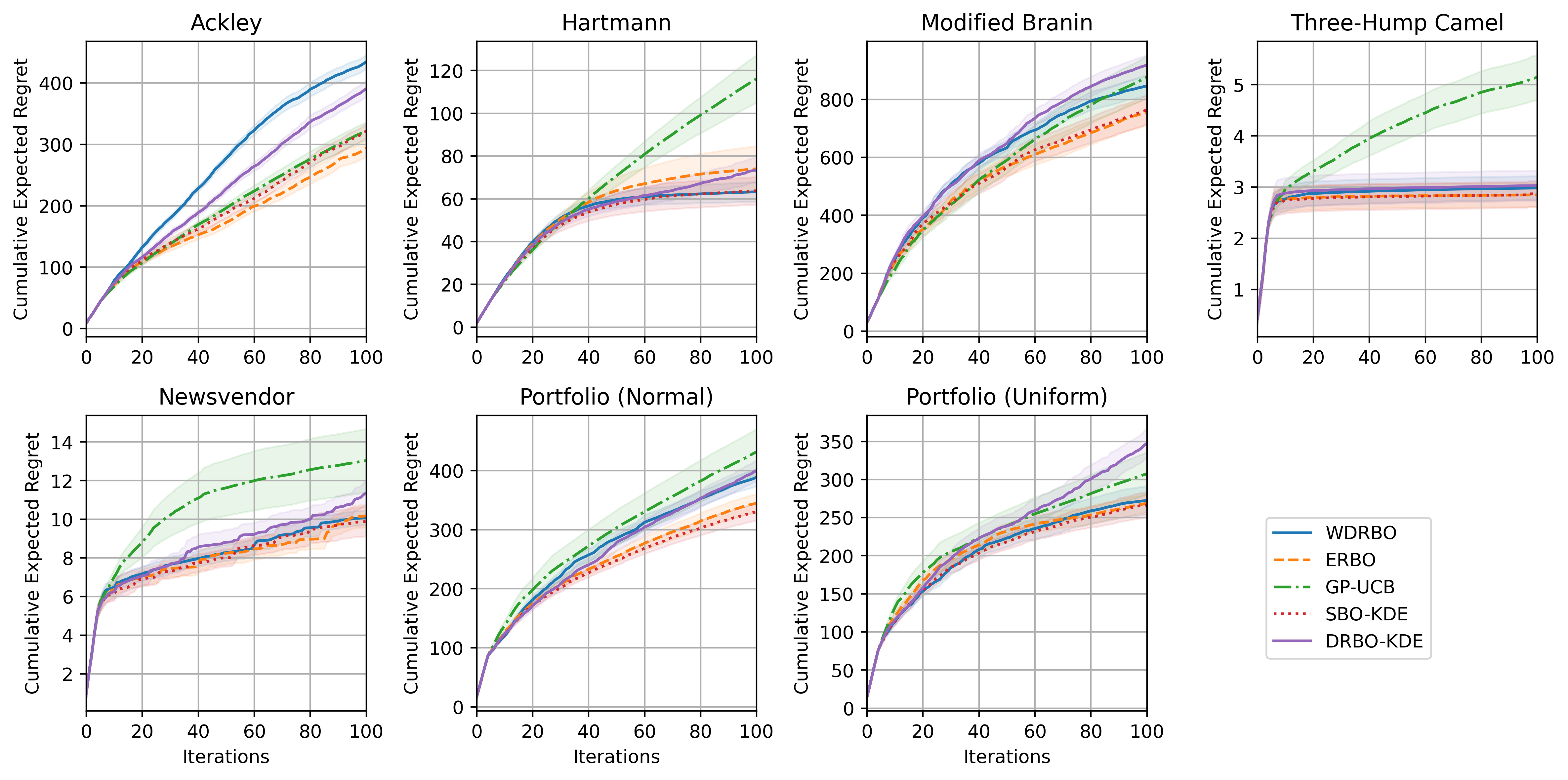}
  \caption{Mean and standard error of the cumulative expected regret.}
  \label{fig:1}
\end{figure*}

We will compare compare the following algorithms:
\textbf{WDRBO}: Data-Driven WDRBO algorithm with robustified acquisition function~\ref{eq:WUCB}, where the center of the Wasserstein ambiguity set is given by the empirical distribution of the observed contexts and the radius is chosen as $\varepsilon_t = O(1/\sqrt{t})$.\\
\textbf{ERBO}: This is equivalent to WDRBO but we set $\varepsilon_t = 0$ in the acquisition function~\ref{eq:WUCB}, i.e. we maximize the empirical risk with respect to the observed contexts $\xt = \arg \max_{x\in\mathcal{X}} \EPhatt \left[ \UCB(x,c) \right]\ $.\\
\textbf{GP-UCB}: Implements the UCB maximization algorithm proposed by~\cite{srinivas2012information} ignoring the context variable in both the definition of the Gaussian process model and in the acquisition function maximization. \\
\textbf{SBO-KDE}: Stochastic BO formulation of~\cite{huang2024stochastic}. An approximate context distribution is estimated from the observed samples by kernel density estimation. The acquisition function maximizes the expectation of the UCB with respect to the empirical distribution of the context obtained by sampling the approximate context distribution (sample average approximation).\\
\textbf{DRBO-KDE}: DR formulation of SBO-KDE proposed by~\cite{huang2024stochastic}. Robustifies the SBO-KDE algorithm by considering DR formulation with a total variation ambiguity set. The ambiguity set is centered on the empirical distribution of the context obtained by sampling from the density estimate.\\
\textbf{DRBO-MMD}: DRBO formulation with MMD ambiguity set of~\cite{kirschner2020distributionally}. The continuous context space is discretized and the UCB is maximized for the worst-case distribution supported on the discrete context space for a given MMD budget. The complexity of the robustified acquisition function scales with the cube of the cardinality of the context support. \\
\textbf{DRBO-MMD Minmax}: Minmax approximate formulation of DRBO-MMD proposed in~\cite{tay2022efficient}. The discretization can be finer as the worst-case sensitivity approximation reduces the computational burden of the method.\\
\textbf{StableOpt}: Implementation of StableOpt~\cite{bogunovic2018adversarially}. Implements a robust acquisition function $x_t = \arg \max_{x\in\mathbb{X}} \min_{c\in C_t} \UCB(x,c)$, where, following~\cite{huang2024stochastic}, the set $C_t$ is chosen at each time $t$ as the set where for each dimension of the context we consider the interval $[\hat{\mu}_t^{c^i} - \hat{\sigma}_t^{c^i},\hat{\mu}_t^{c^i} + \hat{\sigma}_t^{c^i} ]$, with $\hat{\mu}_t^{c^i}$ and $\hat{\sigma}_t^{c^i}$ the empirical mean and variance of the observed contexts.\\

For all the algorithms considered we fixed the value of the UCB trade-off parameter $\beta_t=1.5$. This has been done to be consistent with the engineering practice and earlier works such as~\cite{huang2024stochastic}. 
We consider a set of artificial and real-world problems and different types of context distributions. For each problem and algorithm, we ran $100$ iterations and repeated over $15$ random seeds. Fig.~\ref{fig:1} shows the resulting cumulative expected regret for WDRBO, ERBO, GP-UCB, SBO-KDE, and DRBO-KDE. More details about the specifics of the test problems and the implementations are left in the Appendix. We also leave the results for DRBO-MMD, DRBO-MMD Minmax, and StableOpt to the Appendix as their performance was not competitive with the other methods. The performance of DRBO-MMD is limited by the coarseness of the context discretization which is required to have a computationally tractable inner convex optimization step. The performance of DRBO-MMD Minmax is mainly limited by the worst-case sensitivity that introduces a linear term in the resulting regret bound. StableOpt suffers from the fact that it is solving a robust optimization problem.

We tracked the time required by each algorithm for a 100-iteration-long experiment. We report in Table~\ref{table:1} the computational times in seconds for the Ackley and Branin functions. The computational times are affected both by algorithm specific characteristics, e.g. an inner convex optimization problem is solved at each iteration, and by specific parameters choices, e.g. the discretization grid-size. The reported times have been obtained by running the algorithm on CPU only, as some of the algorithms have not been implemented to exploit the potential speed-ups resulting from running on GPU. For this test we used an Intel(R) Core(TM) i9-9900K@3.60GHz. 
GP-UCB has the smallest computational time as it ignores the context, thus also reducing the regression step complexity. ERBO and SBO-KDE are not robust approaches, the extra time required by SBO-KDE is due to the KDE step. One of the advantage of the proposed WDRBO is that is is able to add robustness against context distribution uncertainty without the large overheads of the other robust methods. The extra computational time required by WDRBO compared to ERBO is related to the calculation of the Lipschitz constant from the UCB expression. The main computational bottleneck for DRBO-MMD and DRBO-KDE is related to the solution of the inner minimization problems.

\begin{table}[ht]
\centering
\begin{tabular}{|l|c|c|}
\hline
$ $ &\textbf{Ackley} & \textbf{Branin} \\ \hline
% \multirow{1}{*}{$ $} & \multicolumn{2}{c|}{\textbf{Ackley ($d_x\!=\!1$, $d_c\!= \!1$)}} & \multicolumn{2}{c|}{\textbf{Branin ($d_x\! = \!2$, $d_c \!= \!2$)}} \\ \cline{2-5}
%  & \textbf{Mean [s]} & \textbf{Std [s]} & \textbf{Mean [s]} & \textbf{Std [s]} \\ \hline
WDRBO (ours) & $44.3\pm 2.2$ & $54.5 \pm 2.6$ \\ \hline
ERBO (ours) & $43.8\pm 1.6$& $45.2 \pm 2.5$ \\ \hline
GP-UCB & $15.7\pm 1.4$ & $15.1 \pm 1.0$\\ \hline
SBO-KDE & $46.7\pm 0.7$ & $49.7\pm 1.5$ \\ \hline
DRBO-KDE & $599.7 \pm 33.0$ & $525.0 \pm 71.7$ \\ \hline
% DRBO-MMD & $644.9 \pm 45.6[s]$ & $131.6 \pm 9.6[s]$ \\ \hline
% DRBO-MMD Minimax & $104.8 \pm 3.9[s]$ & $21.1 \pm 1.8[s]$ \\ \hline
% StableOpt & $77.6 \pm 2.4[s]$ & $64.7 \pm 4.2[s]$ \\ \hline
\end{tabular}
\caption{Mean and standard error of computational times in seconds for the Ackley ($d_x = 1$, $d_c = 1$) and Branin ($d_x = 2$, $d_c = 2$) functions.}\label{table:1}
\end{table}

We can see in Fig.~\ref{fig:1} that the performance of WDRBO, ERBO, and SBO-KDE is extremely compelling, particularly when considering the computational complexity of the other algorithms. We argue that the performance of DRBO-KDE does not justify the extra computation required to solve the inner two-dimensional optimization problem.
While we observe very strong performances for ERBO and SBO-KDE in the data-driven setting, with the smallest computational complexities, we want to highlight that, contrarily to WDRBO and DRBO-KDE, they do not compute a robust solution. This might lead to disappointing performance as shown in the first example where the ambiguity set does not collapse to the true distribution as the number of iterations grows, and for which a robust solution might be preferable. 

\section{CONCLUSIONS AND FUTURE WORK}
\label{sec:conclusions}
In this paper, we introduced Wasserstein Distributionally Robust Bayesian Optimization (WDRBO), a novel algorithm that addresses the challenge of sequential data-driven decision-making under context distributional uncertainty. We developed a computationally tractable algorithm for WDRBO that can handle continuous context distributions, leveraging an approximate reformulation based on Lipschitz bounds of the acquisition function. 
This approach extends the existing literature on Distributionally Robust Bayesian Optimization by providing a principled method to handle continuous context distributions within a Wasserstein ambiguity set, allowing for a flexible and intuitive way to model uncertainty in the context distribution while maintaining computational feasibility.

Our theoretical analysis provides an cumulative expected regret bounds that match state-of-the-art results. Notably, for the data-driven setting, the bound does not require assumptions on the rate of decay of the ambiguity set radius but relies on finite-sample concentration results, making our approach more broadly applicable to real-world situations.
Lastly, we conducted a comprehensive empirical evaluation demonstrating the effectiveness and practical applicability of WDRBO on both synthetic and real-world benchmarks. Our results show that the proposed WDRBO algorithm exhibits promising performance in terms of regret while avoiding computationally expensive inner optimization steps.

The promising results of WDRBO open up exciting opportunities for further research and development in the field of Distributionally Robust Bayesian optimization. Extending the WDRBO framework to risk measures, such as Conditional Value at Risk (CVaR), could broaden its applicability to risk-sensitive domains, such as robotics and finance.

\bibliography{main}

%%%%%%%%%%%%%%%%%%%%%%%%%%%%%%%%%%%%%%%%%%%%%%%%%%%%%%%%%%%%
\section*{Checklist}
 \begin{enumerate}

 \item For all models and algorithms presented, check if you include:
 \begin{enumerate}
   \item A clear description of the mathematical setting, assumptions, algorithm, and/or model. [Yes]
   \item An analysis of the properties and complexity (time, space, sample size) of any algorithm. [Yes]
   \item (Optional) Anonymized source code, with specification of all dependencies, including external libraries. [Yes]
 \end{enumerate}

 \item For any theoretical claim, check if you include:
 \begin{enumerate}
   \item Statements of the full set of assumptions of all theoretical results. [Yes]
   \item Complete proofs of all theoretical results. [Yes]
   \item Clear explanations of any assumptions. [Yes]     
 \end{enumerate}

 \item For all figures and tables that present empirical results, check if you include:
 \begin{enumerate}
   \item The code, data, and instructions needed to reproduce the main experimental results (either in the supplemental material or as a URL). [Yes]
   \item All the training details (e.g., data splits, hyperparameters, how they were chosen). [Yes]
         \item A clear definition of the specific measure or statistics and error bars (e.g., with respect to the random seed after running experiments multiple times). [Yes]
         \item A description of the computing infrastructure used. (e.g., type of GPUs, internal cluster, or cloud provider). [Yes]
 \end{enumerate}

 \item If you are using existing assets (e.g., code, data, models) or curating/releasing new assets, check if you include:
 \begin{enumerate}
   \item Citations of the creator If your work uses existing assets. [Yes]
   \item The license information of the assets, if applicable. [Yes]
   \item New assets either in the supplemental material or as a URL, if applicable. [Yes]
   \item Information about consent from data providers/curators. [Not Applicable]
   \item Discussion of sensible content if applicable, e.g., personally identifiable information or offensive content. [Not Applicable]
 \end{enumerate}

 \item If you used crowdsourcing or conducted research with human subjects, check if you include:
 \begin{enumerate}
   \item The full text of instructions given to participants and screenshots. [Not Applicable]
   \item Descriptions of potential participant risks, with links to Institutional Review Board (IRB) approvals if applicable. [Not Applicable]
   \item The estimated hourly wage paid to participants and the total amount spent on participant compensation. [Not Applicable]
 \end{enumerate}

 \end{enumerate}

 \newpage
\input{Appendix}

\end{document}

%% file: Appendix.tex
\newpage
\onecolumn
\aistatstitle{Wasserstein Distributionally Robust Bayesian Optimization with~Continuous~Context:
Supplementary~Materials}

\section{BACKGROUND ON RKHS AND KERNEL RIDGE REGRESSION}
\label{sec:rkhs}
In this paper, we consider the frequentist perspective and formulate the surrogate model as the solution of a regularized least-squares regression problem in the Reproducing Kernel Hilbert Space (RKHS). A similar formulation can be derived following the Bayesian perspective of Gaussian Process Regression under suitable assumptions on the Gaussian Process prior and observation noise~\cite{kanagawa2018gaussian}.

Consider an RKHS $(\Hk, \langle \cdot, \cdot \rangle_\Hk)$ with reproducing kernel $k: \mathcal{Z} \times \mathcal{Z} \rightarrow \mathbb{R}$. Define the inner product of the RKHS as $f^\top g = \langle f, g \rangle_\Hk$ and the outer product as $fg^\top = f\langle g, \cdot \rangle_\Hk$. Let $\Phi_t := (k(\cdot, z_1), \ldots, k(\cdot, z_t))^\top$ be the feature map of the RKHS for a sequence of points $z_1, \ldots, z_t \in \mathcal{Z}$. Define the kernel matrix $K_t = \Phi_t\Phi_t^\top$ and the covariance operator $V_t = \Phi_t^\top\Phi_t$. The RKHS norm of a function $f \in \Hk$ is defined as $\|f\|_\Hk = \sqrt{f^\top f} = \sqrt{\langle f, f \rangle_\Hk}$. By the reproducing property of the kernel, we have that $f(z) = \langle f, k(\cdot, z) \rangle_\Hk$ for all $z \in \mathcal{Z}$.

With a slight abuse of notation we write the following equality which will be useful in the upcoming derivations
\begin{equation}\label{eq:Vinv_to_sigma}
    \lambda (\Phi_t^\top\Phi_t + \lambda I)^{-1} = I - \Phi_t^\top(\Phi_t\Phi_t^\top + \lambda I)^{-1}\Phi_t\ ,
\end{equation}
where it should be clear from the context that $I$ is either the identity matrix or the identity operator in the RKHS. We also define the short-hand notation $\bar{V}_t = V_t + \lambda I = \Phi_t^\top\Phi_t + \lambda I$.

Given the observed data $\mathcal{D}_t = \{(z_i, y_i)\}_{i=1}^t$, the regularized least-squares regression problem in RKHS is defined as follows:
\begin{equation}
    \min_{\mu \in \Hk} \sum_{i=1}^t (y_i - \mu(z_i))^2 + \lambda \|\mu\|_\Hk^2.
\end{equation}
The solution to this problem is given by:
\begin{equation}
    \begin{aligned}
        \mu_t &= \bar{V}_t^{-1} \Phi_t^\top y_{1:t} \\
        &= (V_t + \lambda I)^{-1} \Phi_t^\top y_{1:t} \\
        &= \Phi_t^\top(\Phi_t\Phi_t^\top + \lambda I)^{-1} y_{1:t}
    \end{aligned}
\end{equation}
where $y_{1:t} = (y_1, \ldots, y_t)^\top$ is the vector of observed responses.
% We define
% $K_t := \Phi_t\Phi_t^\top$.
By the representation theorem, we can compute $\mu_t$ at some new point $z \in \mathcal{Z}$ as follows:
\begin{equation}
    \begin{aligned}
        \mu_t(z) &= \langle \mu_t, k(\cdot, z) \rangle_\Hk \\
        &= \Phi_t^\top(\Phi_t\Phi_t^\top + \lambda I)^{-1} y_{1:t} \\
        &= k_t(z)^\top (K_t + \lambda I)^{-1} y_{1:t},
    \end{aligned}
\end{equation}
where $k_t(z) = (k(z, z_1), \ldots, k(z, z_t))^\top$ is the vector of kernel evaluations at $z$.

We can also compute 
\begin{equation}
    \begin{aligned}
        \sigma_t^2(z) &:= \|k(\cdot, z)\|_{\bar{V}_t^{-1}}^2 \\
        &= \frac{1}{\lambda} \left( k(z, z) - k_t(z)^\top(K_t + \lambda I)^{-1}k_t(z) \right),
    \end{aligned}
\end{equation}
where $\|k(\cdot, z)\|_{\bar{V}_t^{-1}}^2 = k_t(z)^\top(\Phi_t^\top\Phi_t + \lambda I)^{-1}k_t(z)$, and we use \eqref{eq:Vinv_to_sigma} to get the final equation.

We are using the notation $\mu_t(z)$ and $\sigma_t^2(z)$ to align with the Gaussian Process Regression literature~\cite{kanagawa2018gaussian}, where $\mu_t(z)$ and $\sigma_t^2(z)$ would represent the mean and variance of the Gaussian Process posterior at $z$ respectively.

\subsection{Kernels that satisfy the Lipschitz condition in Assumption 1}
Assumption~\ref{ass:Lip} is satisfied for commonly used kernels. For example, it is satisfied with $L=1$ for the squared exponential kernel and the Matérn kernel for $\nu=3/2$~\citep[Proposition 2]{van2022kernel}. 
In fact, all smooth, positive definite, stationary kernels that have zero derivatives at zero satisfy Assumption~\ref{ass:Lip}. This, in turn, implies that Assumption~\ref{ass:Lip} is satisfied for Mat\'ern kernels with $\nu=p+1/2$, for $p=1,2,\dots$.

\begin{lem}\label{app_lem:kernels_that_satisfy_ass_1} 
Let $k$ be a positive definite, stationary kernel such that $k(x,x')=r(\|x-x'\|)$, for some function $r:\mathbb{R} \to \mathbb{R}$ that is continuously twice differentiable in a neighborhood of the origin with first derivative $r^{(1)}(0)=0$. 
% For such a kernel, the positive definiteness property implies that the second derivative$r^{(2)}(0) \leq 0$. Therefore, there exists $\epsilon > 0$ such that $r^{(2)}(t) \leq 0$ for all $t \in [0,\epsilon]$. 
Then, the kernel-induced distance 
\[ d(x,x'):=\sqrt{k(x,x)+k(x',x')-k(x,x')-k(x',x)}\le M \|x-x'\|, \] 
for some constant $M>0$.
\end{lem}

\begin{proof} 
Replacing $k$ with $r$ in the expression of $d(x,x')$ we write 
\[ d(x,x')=\sqrt{2r(0)-2r(\|x-x'\|)}. \] 

Pick any $\epsilon>0$. Let's consider two cases:

Case 1: For $\|x-x'\| \geq \epsilon$, by the positive definite property, we have $|r(\|x-x'\|)| \leq r(0)$ for any $x,x'$. Therefore:
\[ d(x,x') \leq 2\sqrt{r(0)} \leq 2\sqrt{r(0)}\, \|x-x'\|/\epsilon = M_1 \|x-x'\|, \]
where $M_1 = 2\sqrt{r(0)}/\epsilon$.

Case 2: For $\|x-x'\| \leq \epsilon$, using the Taylor remainder formula and the fact that $r^{(1)}(0)=0$:
\[ r(\|x-x'\|) = r(0) + \frac{\|x-x'\|^2}{2} r^{(2)}(s), \]
for some $s \in [0, \|x-x'\|]$.
Since $r(\|x-x'\|)\le r(0)$, we have $\|x-x'\|^2 r^{(2)}(s)\le 0$, which, in turn, implies that $r^{(2)}(s)\le 0$. As a result,
\[ d(x,x') = \sqrt{-\|x-x'\|^2 r^{(2)}(s)} = \|x-x'\| \sqrt{-r^{(2)}(s)} \leq \|x-x'\| \max_{s\in[0,\epsilon]}\sqrt{-r^{(2)}(s)} = M_2 \|x-x'\|. \]

Since the function has a continuous second derivative in the interval $[0,\epsilon]$ and $r^{(2)}(s) \leq 0$ for all $s \in [0,\epsilon]$, the maximum $M_2 = \max_{s\in[0,\epsilon]}\sqrt{-r^{(2)}(s)}$ is well-defined and finite.

The result follows by taking
\[ M = \max\{M_1, M_2\} = \max\left\{\frac{2\sqrt{r(0)}}{\epsilon}, \max_{s\in[0,\epsilon]}\sqrt{-r^{(2)}(s)}\right\} \]
\end{proof}

\section{MAIN PROOFS}\label{app_sec:main_proofs}

We can state here a well-known result from the Wasserstein DR optimization literature~\cite{kuhn2019wasserstein,gao2024wasserstein}.
\begin{lem}
    \label{lem:wdro_appendix}
    Consider a function $g: {\Xi} \rightarrow \mathbb{R}$ that is $L^g$-Lipschitz, i.e. $|g(\xi) - g(\xi')| \leq L^g \|\xi - \xi'\|$, for all $\xi, \xi' \in {\Xi}$. Let $\mathcal{B}^{\varepsilon}(\hat{\mathcal{P}})$ be a Wasserstein ambiguity set defined as a ball of radius $\varepsilon$ in the Wasserstein distance centered at the distribution $\hat{\mathcal{P}}$. Then, 
    \begin{equation}
        \sup_{\mathcal{Q} \in \mathcal{B}^{\varepsilon}(\hat{\mathcal{P}})} \mathbb{E}_{\xi \sim \mathcal{Q}} [g(\xi)] \leq \mathbb{E}_{\xi \sim \hat{\mathcal{P}}} [g(\xi)] + \varepsilon L^g\ ,
    \end{equation}
    Similarly, 
    \begin{equation}
        \inf_{\mathcal{Q} \in \mathcal{B}^{\varepsilon}(\hat{\mathcal{P}})} \mathbb{E}_{\xi \sim \mathcal{Q}} [g(\xi)] \geq \mathbb{E}_{\xi \sim \hat{\mathcal{P}}} [g(\xi)] - \varepsilon L^g\ .
    \end{equation}
\end{lem}

As a consequence of Lemma~\ref{lem:wdro_appendix}, we can state the following result. 
\begin{lem}
    % \label{cor:wdro}
    Let $f: \mathcal{X} \times \mathcal{C} \rightarrow \mathbb{R}$ be a function that is $L^f_c(x)$-Lipschitz in the context space, i.e. $|f(x, c) - f(x, c')| \leq L^f_c(x) \|c - c'\|$, for all $c, c' \in \mathcal{C}$. Let $\mathcal{B}^{\varepsilon}(\hat{\mathcal{P}})$ be a Wasserstein ambiguity set defined as a ball of radius $\varepsilon$ in the Wasserstein distance centered at the distribution $\hat{\mathcal{P}}$. Then, for any $x\in\mathcal{X}$ and for any distribution $\tilde{\mathcal{P}} \in \mathcal{B}^{\varepsilon}(\hat{\mathcal{P}})$, we have that
    \begin{equation}
        | \mathbb{E}_{c \sim \tilde{\mathcal{P}}}[f(x, c)] - \mathbb{E}_{c \sim \hat{\mathcal{P}}}[f(x, c)] | \leq \varepsilon L^f_c(x)\ .
    \end{equation}
\end{lem}

In order to provide rates for the cumulative expected regret we want to provide high probability bounds for the Lipschitz constants $L^{\UCB}(x)$.

\begin{lem}[Lemma~\ref{lem:B_bar} in the main text]\label{app_lem:B_bar}
    Let $0< \delta < 1$ be a failure probability and let 
    \begin{align}
    \bar{B}_t := \lambda^{-\frac{1}{2}}R \sqrt{2\log\left(\frac{\det(I + \lambda^{-1}K_{t-1})^\frac{1}{2}}{\delta}\right)} + B \leq \lambda^{-\frac{1}{2}} \left( \left( R \sqrt{2 \log \frac{1}{\delta} } + R \sqrt{2\gamma_t} \right) + B \right) .
    \end{align}
    
    Then, with probability $1-\delta$ for all $t\geq 1$ we have 
    $\|\mu_{t-1}\|_{\mathcal{H}_k} \leq \bar{B}_t.$
    Further, if Assumption~\ref{ass:Lip} holds (i.e., $\|k(\cdot,z)-k(\cdot,z')\|_{\mathcal{H}_k} \leq L \|z-z'\|$), we have:
    
    (i) With probability $1-\delta$, for any $z, z' \in \mathcal{X} \times \mathcal{C}$:
    $|\mu_{t-1}(z) - \mu_{t-1}(z')| \leq \bar{B}_t L \|z-z'\|.$
    
    (ii) For any $z, z' \in \mathcal{X} \times \mathcal{C}$:
    $|\beta_t\sigma_{t-1}(z) - \beta_t\sigma_{t-1}(z')| \leq \beta_t \lambda^{-\frac{1}{2}} L \|z-z'\| = \bar{B}_t L \|z-z'\|.$
    
    Therefore, with probability $1-\delta$, the UCB function is Lipschitz continuous with constant:
    $L^{\UCB} \leq 2\bar{B}_t L.$
\end{lem}

\begin{proof}
    The UCB is defined as:
    $$\UCB(z) = \mu_{t-1}(z) + \beta_t \sigma_{t-1}(z),$$
    where $z = (x,c) \in \mathcal{X} \times \mathcal{C}$.
    
    We have
    \begin{equation}
    \mu_t(z) = \langle (\lambda I + V_t)^{-1} \Phi_t^\top y_{1:t}, k(\cdot, z) \rangle,
    \end{equation}
    and
    \begin{align}
        \|\mu_t\|_{\mathcal{H}_k} &= \|(\lambda I + V_t)^{-1} \Phi_t^\top y_{1:t}\|_{\mathcal{H}_k} \\
        &= \|(\lambda I + V_t)^{-1} \Phi_t^\top (f(z_{1:t})+\eta_{1:t})\|_{\mathcal{H}_k} \\
        &= \|(\lambda I + V_t)^{-1} \Phi_t^\top (\Phi_t f + \eta_{1:t})\|_{\mathcal{H}_k} \\
        &\leq \|(\lambda I + V_t)^{-1} \Phi_t^\top \Phi_t f\|_{\mathcal{H}_k} + \|(\lambda I + V_t)^{-1} \Phi_t^\top \eta_{1:t}\|_{\mathcal{H}_k} \\
        &= \|(\lambda I + V_t)^{-1} V_t f\|_{\mathcal{H}_k} + \|(\lambda I + V_t)^{-1} \Phi_t^\top \eta_{1:t}\|_{\mathcal{H}_k} \\
        &\leq \|(\lambda I + V_t)^{-1} V_t\|_{\mathcal{H}_k} \|f\|_{\mathcal{H}_k} + \|(\lambda I + V_t)^{-\frac{1}{2}}\|_{\mathcal{H}_k} \|(\lambda I + V_t)^{-\frac{1}{2}} \Phi_t^\top \eta_{1:t}\|_{\mathcal{H}_k} \\
        &\leq B + \lambda^{-\frac{1}{2}} \|(\lambda I + V_t)^{-\frac{1}{2}} \Phi_t^\top \eta_{1:t}\|_{\mathcal{H}_k},
    \end{align}
    where we used the assumption that $\|f\|_{\mathcal{H}_k} \leq B$ and the fact that $\|(\lambda I + V_t)^{-1} V_t\|_{\mathcal{H}_k} \leq 1$ for $\lambda \geq 0$.

    Applying Corollary 3.6 of \cite{abbasi2013online}, with probability at least $1-\delta$ and for all $t \geq 1$, we have
    \begin{align}
        \|(\lambda I + V_t)^{-\frac{1}{2}} \Phi_t^\top \eta_{1:t}\|_{\mathcal{H}_k} \leq R \sqrt{2\log\left(\frac{\det(I + \lambda^{-1}V_{t})^\frac{1}{2}}{\delta}\right)} = R \sqrt{2\log\left(\frac{\det(I + \lambda^{-1}K_{t})^\frac{1}{2}}{\delta}\right)}.
    \end{align}
    
    Thus, we obtain
    \begin{align}\label{eq_RKHS_norm_mu}
        \|\mu_t\|_{\mathcal{H}_k} \leq \lambda^{-\frac{1}{2}} R \sqrt{2\log\left(\frac{\det(I + \lambda^{-1}K_{t})^\frac{1}{2}}{\delta}\right)} + B = \bar{B}_t.
    \end{align}
    
    If Assumption~\ref{ass:Lip} holds, for any $z, z' \in \mathcal{X} \times \mathcal{C}$, we have:
    \begin{align}
        |\mu_{t-1}(z) - \mu_{t-1}(z')| &= |\langle \mu_{t-1}, k(\cdot, z) - k(\cdot, z') \rangle_{\mathcal{H}_k}| \\
        &\leq \|\mu_{t-1}\|_{\mathcal{H}_k} \cdot \|k(\cdot, z) - k(\cdot, z')\|_{\mathcal{H}_k} \\
        &\leq \bar{B}_t \cdot L \|z - z'\| 
    \end{align}
    where we used the Cauchy-Schwarz inequality and Assumption~\ref{ass:Lip}.
    
    For the term $\beta_t \sigma_{t-1}(z)$, we need to analyze the Lipschitz property of $\sigma_{t-1}(z)$. We know by the Woodbury identity:
    \begin{equation}
        \sigma_{t-1}^2(z) = \langle k(\cdot, z), (\lambda I + V_{t-1})^{-1} k(\cdot, z) \rangle_{\mathcal{H}_k} = \frac{1}{\lambda} \left( k(z, z) - k_{t-1}(z)^\top(K_{t-1} + \lambda I)^{-1}k_{t-1}(z) \right)
    \end{equation}

    This gives us:
    \begin{equation}
        \sigma_{t-1}(z) = \|(\lambda I + V_{t-1})^{-1/2} k(\cdot, z)\|_{\mathcal{H}_k}
    \end{equation}

    Now, for the Lipschitz property, by the triangle inequality:
    \begin{align}
        |\sigma_{t-1}(z) - \sigma_{t-1}(z')| &\leq \|(\lambda I + V_{t-1})^{-1/2}(k(\cdot, z) - k(\cdot, z'))\|_{\mathcal{H}_k} \\
        &\leq \|(\lambda I + V_{t-1})^{-1/2}\|_{op} \cdot \|k(\cdot, z) - k(\cdot, z')\|_{\mathcal{H}_k} \\
        &\leq \lambda^{-\frac{1}{2}} \cdot L \|z - z'\|
    \end{align}
    where we used the fact that $\|(\lambda I + V_{t-1})^{-1/2}\|_{op} \leq \lambda^{-\frac{1}{2}}$ and Assumption~\ref{ass:Lip}.

    Therefore:
    \begin{equation}
        |\beta_t\sigma_{t-1}(z) - \beta_t\sigma_{t-1}(z')| \leq \beta_t \lambda^{-\frac{1}{2}} L \|z - z'\|
    \end{equation}
    
    Note that, recalling the definition of $\beta_t$:
    \begin{align}
        \beta_t := R \sqrt{2\log\left(\frac{\det(I + \lambda^{-1}K_{t-1})^\frac{1}{2}}{\delta}\right)} + \lambda^{\frac{1}{2}}B,
    \end{align}
    we can observe that $\beta_t \lambda^{-\frac{1}{2}} = \bar{B}_t$.
    
    Combining the results, with probability $1-\delta$, we have:
    \begin{align}
        |UCB(z) - UCB(z')| &= |\mu_{t-1}(z) + \beta_t\sigma_{t-1}(z) - \mu_{t-1}(z') - \beta_t\sigma_{t-1}(z')| \\
        &\leq |\mu_{t-1}(z) - \mu_{t-1}(z')| + |\beta_t\sigma_{t-1}(z) - \beta_t\sigma_{t-1}(z')| \\
        &\leq \bar{B}_t L \|z - z'\| + \beta_t \lambda^{-\frac{1}{2}} L \|z - z'\| \\
        &= (\bar{B}_t L + \bar{B}_t L) \|z - z'\| \\
        &= 2\bar{B}_t L \|z - z'\|
    \end{align}
    where we used the fact that $\beta_t \lambda^{-\frac{1}{2}} = \bar{B}_t$.

    Thus, with probability $1-\delta$, for all $x\in\mathcal{X}$ the Lipschitz constant of the UCB function with respect to the context is bounded by:
    \begin{equation}
        L^{\UCB}(x) \leq 2\bar{B}_t L
    \end{equation}
    which concludes the proof.
\end{proof}

\begin{thm}[Instantaneous expected regret - \textbf{Thm.~\ref{thm:inst_regret} in the main text}]\label{thm:inst_regret_appendix}
    Let Assumption~\ref{ass:Lip} hold. Fix a failure probability $0<\delta<1$.
    With probability at least $1-\delta$, for all $t\ge 1$ the instantaneous expected regret can be bounded by
    \begin{equation}
        r_t \leq \EPstart \left[ 2 \beta_t \sigma_{t-1}(\xt, c) \right] + 2 \varepsilon_t L^{\UCB}(\xstar_t)
    \end{equation} 
\end{thm}

\begin{proof}
    Recall that the benchmark solution $\xstar_t$ is the optimal solution to the stochastic optimization problem at time-step $t$, given access to the true function $f$ and context distribution $\Pstar_t$
    \begin{equation*}
    \xstar_t = \arg \max_{x\in\mathcal{X}} \EPstart{[f(x, c)]} \ .
    \end{equation*}
    Whereas $\xt$ is the solution to the robustified UCB acquisition function as given in \eqref{eq:WUCB}
    \begin{equation*}
    \xt = \arg \max_{x\in\mathcal{X}} \EPhatt \left[ \UCB(x,c) \right] - \varepsilon_t L^{\UCB}(x) \ ,
    \end{equation*}
    From the definition of instantaneous regret, we can write:
    \begin{equation}
        \begin{aligned}
            r_t &= \EPstart{[\fxtstar]} - \EPstart{[\fxt]}\\
            & \markthis{(i)}{proof_inst_regret_1}{\leq} \EPstart\left[\UCB(\xstar_t,c)\right] -\EPstart\left[\LCB(\xt,c)\right]\\
            & \markthis{(ii)}{proof_inst_regret_2}{\leq} \EPhatt\left[\UCB(\xstar_t,c) \right] + \varepsilon_t L^{\UCB}(\xstar_t) -\EPstart\left[\LCB(\xt,c)\right]\\
            & \markthis{(iii)}{proof_inst_regret_3}\leq \EPhatt\left[\UCB(\xt,c) \right] - \varepsilon_t L^{\UCB}(\xt) + 2 \varepsilon_t L^{\UCB}(\xstar_t) -\EPstart\left[\LCB(\xt,c)\right] \\
            & \markthis{(iv)}{proof_inst_regret_4}\leq \EPstart\left[\UCB(\xt,c) \right] + 2 \varepsilon_t L^{\UCB}(\xstar_t) -\EPstart\left[\LCB(\xt,c)\right] \\
            & \markthis{(v)}{proof_inst_regret_5}\leq \EPstart [\mu_{t-1}(\xt,c) + \beta_t \sigma_{t-1}(\xt, c) - \mu_{t-1}(\xt,c) + \beta_t \sigma_{t-1}(\xt, c) ] + 2 \varepsilon_t L^{\UCB}(\xstar_t) \\
            & = \EPstart \left[ 2 \beta_t \sigma_{t-1}(\xt, c) \right] + 2 \varepsilon_t L^{\UCB}(\xstar_t) \ .
        \end{aligned}
    \end{equation}
    
Where \ref{proof_inst_regret_1} holds with probability $1-\delta$ by definition of UCB and LCB, with $\beta_t$ chosen as in Lemma~\ref{lem:rkhs_UCB_LCB}. \ref{proof_inst_regret_2} holds by applying Lemma 2 under the assumption that $\Pstar_t \in \mathcal{B}^{\varepsilon_t}(\Phat_t)$. In \ref{proof_inst_regret_3} we add and subtract $\varepsilon_t L^{\UCB}(\xstar)$ and use the fact that  $\xt$ is the maximizer of the acquisition function~\eqref{eq:WUCB}. The inequality \ref{proof_inst_regret_4} follows from another application of Lemma 2, and finally \ref{proof_inst_regret_5} follows again from the definitions of the UCB and the LCB.
\end{proof}

\begin{thm}[cumulative expected regret - \textbf{Thm.~\ref{thm:cum_regret} in the main text}]\label{thm:cum_regret_appendix}
    Let Assumption~\ref{ass:Lip} hold. Fix a failure probability $0<\delta<1$.
     With probability at least $1-2\delta$, the cumulative expected regret after $T$ steps can be bounded as:
      \begin{equation}
          R_T \leq 4 \beta_T \sqrt{ T 
        \left(\gamma_T + 4 \log \frac{6}{\delta}\right)} + \sum_{t=1}^{T} \varepsilon_t 2 L^{\UCB}(\xstar_t)\ ,
      \end{equation}
      where $\gamma_{T}$ is the \textit{maximum information gain}~\citep{srinivas2009gaussian, chowdhury2017kernelized, vakili2021information} at time $T$, which is defined as
      $$\gamma_t := \sup_{z_1,z_2,\dots,z_t} \log \det \left( I + \lambda^{-1}K_{t-1} \right)\ .$$
\end{thm}
  
\begin{proof}
    Starting from the definition of cumulative expected regret: 
    \begin{equation}
        \begin{aligned}
            R_T &= \sum_{t=1}^{T} r_t = \sum_{t=1}^{T} \EPstart{[\fxtstar]} - \EPstart{[\fxt]} \\
            & \markthis{(i)}{proof_cum_regret_1}\leq \sum_{t=1}^{T} \EPstart \left[ 2 \beta_t \sigma_{t-1}(\xt, c) \right] + \sum_{t=1}^{T} \varepsilon_t 2 L^{\UCB}(\xstar_t) \\
            & \leq 2 \beta_T \sum_{t=1}^{T} \EPstart \left[  \sigma_{t-1}(\xt, c) \right] + \sum_{t=1}^{T} \varepsilon_t 2 L^{\UCB}(\xstar_t) \\
            & \markthis{(ii)}{proof_cum_regret_2}\leq 2 \beta_T \sqrt{ T \sum_{t=1}^{T} \left(\EPstart \left[ \sigma_{t-1}(\xt, c) \right]\right)^2} + \sum_{t=1}^{T} \varepsilon_t 2 L^{\UCB}(\xstar_t)\\
            & \markthis{(iii)}{proof_cum_regret_3}\leq 2 \beta_T \sqrt{ T \sum_{t=1}^{T} \EPstart \left[ \sigma_{t-1}(\xt, c)^2 \right]} + \sum_{t=1}^{T} \varepsilon_t 2 L^{\UCB}(\xstar_t)\\
        \end{aligned}
    \end{equation}

    Where the inequality \ref{proof_cum_regret_1} follows from Theorem~\ref{thm:inst_regret_appendix}, \ref{proof_cum_regret_2} follows from the Cauchy-Schwarz inequality, and~\ref{proof_cum_regret_3} follows from Jensen's inequality.

    We can now apply the concentration of conditional mean result from Lemma 7 of~\cite{kirschner2020distributionally} (see also Lemma 3 of~\cite{kirschner2018information}), with probability at least $1-\delta$ we obtain for all $T$:
    \begin{equation}
        \begin{aligned}
            &\sum_{t=1}^{T} \EPstart \left[ \sigma_{t-1}(\xt, c)^2 \right]\\
            &\markthis{(i)}{proof_ccmean_1}\leq\ 2\sum_{t=1}^{T} \sigma_{t-1}(\xt, c_t)^2 + 8 \log \frac{6}{\delta}\\
            &\markthis{(ii)}{proof_ccmean_2}\leq\ 2\sum_{t=1}^{T} 2\log(1+\sigma_{t-1}(\xt, c_t)^2) + 8 \log \frac{6}{\delta}\\
            &\markthis{(iii)}{proof_ccmean_3}\leq\ 4\gamma_T + 16 \log \frac{6}{\delta}
        \end{aligned}
    \end{equation}
    where \ref{proof_ccmean_1} follows from Lemma 7 of ~\cite{kirschner2020distributionally} noting that $k(z,z')\leq 1$ by assumption, \ref{proof_ccmean_2} follows from the fact that $x\leq 2 a\log(1+x)$ for all $x\in [0,a]$, and \ref{proof_ccmean_3} follows from the definition of maximum information gain.    By substituting this result into the cumulative regret expression we get with probability $1-2\delta$:
    \begin{equation}
    \begin{aligned}
        R_T\leq&\ 4 \beta_t \sqrt{ T 
        \left(\gamma_T + 4 \log \frac{6}{\delta}\right)} + \sum_{t=1}^{T} \varepsilon_t 2 L^{\UCB}(\xstar_t) \ , \\
    \end{aligned}
    \end{equation}
    which concludes the proof.
\end{proof}

\begin{cor}[Corollary~\ref{cor:wdrbo_regret} in the main text -- General WDRBO Regret Order]\label{cor:wdrbo_regret_appendix}
    Let $0< \delta < 1$ be a failure probability and let Assumption~\ref{ass:Lip} hold. Then, with probability $1-2\delta$, the cumulative expected regret is of the order of
    \[ R_T = \tilde{O}\left(\sqrt{T}\gamma_T+ \sqrt{\gamma_T}\sum_{t=1}^{T} \varepsilon_t\right) \ .
\]
For the Squared Exponential kernel, this reduces to
    \[ R_T = \tilde{O}\left(\sqrt{T} + \sum_{t=1}^{T} \varepsilon_t\right) \ ,
\]
where $\tilde{O}$ omits logarithmic terms. 
\end{cor}
\begin{proof}
    We can combine the results of Theorem 5 and Lemma 6 to write
    \begin{equation}
    \begin{aligned}
        R_T\leq&\ 4 \beta_t \sqrt{ T \gamma_T + 4 \log \frac{6}{\delta}} + \sum_{t=1}^{T} \varepsilon_t 2 L^{\UCB}(\xstar_t) \\
        \leq&\ 4 \beta_t \sqrt{ T \gamma_T + 4 \log \frac{6}{\delta}} + \sum_{t=1}^{T} \varepsilon_t 4 \lambda^{-\frac{1}{2}} \left( \left( R \sqrt{2 \log \frac{1}{\delta} } + R \sqrt{2\gamma_t} \right) + B \right) L \\
        \leq&\ 4 (R\sqrt{2\log(1/\delta)} + R\sqrt{2\gamma_T}+\lambda^{\frac{1}{2}}B) \sqrt{ T \gamma_T + 4 \log \frac{6}{\delta}} + \sum_{t=1}^{T} \varepsilon_t 4 \lambda^{-\frac{1}{2}} \left( \left( R \sqrt{2 \log \frac{1}{\delta} } + R \sqrt{2\gamma_t} \right) + B \right) L \\
        =&\ \mathcal{O}(\gamma_T \sqrt{T} + \sqrt{\lambda \gamma_T T}) + \sum_{t=1}^{T} \varepsilon_t \mathcal{O}(\lambda^{-\frac{1}{2}} \sqrt{\gamma_T }) \\
        =&\ \mathcal{O}(\gamma_T \sqrt{T} + \lambda^{-\frac{1}{2}}\sqrt{\gamma_T }\sum_{t=1}^{T} \varepsilon_t )\ , \\
    \end{aligned}
    \end{equation}
    which proves the first statement.
    
    The maximum information gain for the Squared Exponential kernel can be bounded as~\cite{vakili2021information}: $$\gamma_t\leq\mathcal{O}(\log^{d+1}(t))\ .$$
    
    Thus, the rate for the cumulative expected regret is
    $$R_T = \tilde{\mathcal{O}}\left(\sqrt{T} + \sum_{t=1}^{T} \varepsilon_t\right)$$
\end{proof}

\subsection{Proofs of the data-driven setting}
We can now show how this result translates to the data-driven formulation, and provide a rate for the cumulative expected regret.

Define the class of functions $$\mathcal{U}(m,b,s)=\left\{h: h(z) = \mu(z) + \beta \sigma(z)\ , \ \|\mu\|_\Hk\leq m,\ \beta\leq b,\ 
 \sigma(z) \in \bs{\sigma}_s\right\}\ ,$$ 
 where 
$\bs{\sigma}_s = \left\{\sigma: 
\sigma(z) = \|\mathcal{M} k(\cdot, z)\|_\Hk \text{ with } \| \mathcal{M} \|_{op}\leq s^{-\frac{1}{2}} \right\}\ .$ Following this definition, the set $\mathcal{U}(\bar{B}_t,\bar{B}_t,\lambda)$ contains the UCB functions in the form of~\eqref{eq:UCB_definition}, which follows from the definitions of $\mu_{t}$, $\sigma_t$, and $\beta_t$ given in~\eqref{eq:mu},~\eqref{eq:sigma},~\eqref{eq:beta_definition} respectively. The RKHS-norm bound on $\mu_{t-1}$ follows from~\eqref{eq_RKHS_norm_mu}, the bound on $\beta_t$ follows from Lemma~\ref{lem:rkhs_UCB_LCB}, and the condition $\sigma_{t-1}(z) \in \bs{\sigma}^\lambda$ follows from the fact that the minimum singular value of $\lambda I + V_{t-1}$ is larger than $\lambda$.
Let $\mathcal{N}_{\infty}(\rho,\mathcal{U}(m,b,s))$ be the $\rho$-covering number with respect to the $\infty$-norm of the set $\mathcal{U}(m,b,s)$.

The following result is an adaptation of Corollary~2 of~\cite{gao2023finite}. Let $\diamX$, $\diamC$ denote the diameters of the sets $\mathcal{X},\mathcal{C}$ respectively.

\begin{lem}\label{lem:adapted_gao}
  Let $0<\delta<1$ be a failure probability. Let 
  \[\varepsilon_t(\rho)=\sqrt{2\diamC^2\frac{ \log 1/\delta+ d_x\log(1+2\rho^{-1}\diamX)+\log \mathcal{N}_{\infty}(\rho,\mathcal{U}(m,b,s))}{t}}.
  \] With probability at least $1-\delta$
  \begin{equation}
    \forall h\in\mathcal{U}(m,b,s), \forall x\in\mathcal{X}:\quad   \mathbb{E}_{c \sim \Pstar}[h(x, c)] \le  \mathbb{E}_{c \sim \Phat_t}[h(x, c)]  + \varepsilon_t(\rho) L^h_c(x) + (1+L_x)\rho.
  \end{equation}
  where $L^h_c(x)$ denotes the Lipschitz constant of function $h(x,\cdot)$ with respect to the context $c$, and $L_x$ is such that $\forall\ h\ \in \mathcal{U}(m,b,s),\ x,\ \tilde{x}\ \in \mathcal{X},\ c\in\mathcal{C}$ we have that $|h(x,c)-{h}(\tilde{x},c)|\leq L_x \|x-\tilde{x}\|$.
\end{lem}
\begin{proof}
  Let $h\in\mathcal{U}(m,b,s)$, $x\in\mathcal{X}$.  Following the notation of~\cite{gao2023finite}, let for any $\varepsilon>0$
  \[
    \mathcal{R}_{\mathcal{P}}(\varepsilon,h(x,\cdot))=\sup_{ \mathcal{Q} \in\mathcal{B}^{\varepsilon}(\mathcal{P})}\mathbb{E}_{c\sim \mathcal{Q}}(h(x,c))-\mathbb{E}_{c\sim \mathcal{P}}(h(x,c)).
  \]
  Using the Lipschitz bound from Lemma~\ref{lem:B_bar}, it holds that
  \begin{equation}\label{app_eq:Lip_bound_on_R}
    \mathcal{R}_{\mathcal{P}}(\varepsilon,h(x,\cdot))\le \varepsilon L^h_c(x)
  \end{equation}

  Let $\tilde{h}\in\mathcal{U}(m,b,s),\,\tilde{x}\in\mathcal{X}$. Following the proof of Corollary~2 in~\cite{gao2023finite}, we have
  \[
    \mathcal{R}_{\mathcal{P}}(\varepsilon,-h(x,\cdot))-\mathcal{R}_{\mathcal{P}}(\varepsilon,-\tilde{h}(\tilde{x},\cdot))\le \sup_{ \mathcal{Q} \in\mathcal{B}^{\varepsilon}(\mathcal{P})}|\mathbb{E}_{c\sim \mathcal{Q}}(\tilde{h}(\tilde{x},c))-\mathbb{E}_{c\sim \mathcal{Q}}(h(x,c))|.
  \]
  Using the decomposition
  \begin{align*}
    |h(x,c)-\tilde{h}(\tilde{x},c)|&=|h(x,c)-\tilde{h}(x,c)+\tilde{h}(x,c)-\tilde{h}(\tilde{x},c)|\\
    &\le \|h-\tilde{h}\|_{\infty}+L_x\|x-\tilde{x}\|,
  \end{align*}
  we obtain
  \begin{equation}\label{app_eq:helper_R_equation}
    \mathcal{R}_{\mathcal{P}}(\varepsilon,-h(x,\cdot))-\mathcal{R}_{\mathcal{P}}(\varepsilon,-\tilde{h}(\tilde{x},\cdot))\le   \|h-\tilde{h}\|_{\infty}+L_x\|x-\tilde{x}\|
  \end{equation}
  Similarly
  \begin{equation}\label{app_eq:helper_R_equation_b}
    \begin{aligned}
      |\mathbb{E}_{c\sim \Pstar}(\tilde{h}(\tilde{x},c))-\mathbb{E}_{c\sim \Pstar}(h(x,c))|&\le \|h-\tilde{h}\|_{\infty}+L_x\|x-\tilde{x}\|\\
      |\mathbb{E}_{c\sim \Phat_t}(\tilde{h}(\tilde{x},c))-\mathbb{E}_{c\sim \Phat_t}(h(x,c))|&\le \|h-\tilde{h}\|_{\infty}+L_x\|x-\tilde{x}\|
    \end{aligned}
  \end{equation}

  Let $\mathcal{X}_{\rho}$ be a covering of $\mathcal{X}$ of precision $\rho$ with respect to the Euclidean norm and $\mathcal{U}_{\rho}(m,b,s)$ be a covering of $\mathcal{U}(m,b,s)$ of precision $\rho$ with respect to the infinity norm. We have that
  $|\mathcal{X}_{\rho}|\le (1+2\diamX \rho^{-1})^{{d_x}}$~\cite[Ch. 4]{vershynin2018high}, while by definition $|\mathcal{U}_{\rho}(m,b,s)|=\mathcal{N}_{\infty}(\rho,\mathcal{U}(m,b,s))$. 

  By the definition of the coverings and~\eqref{app_eq:helper_R_equation},~\eqref{app_eq:helper_R_equation_b}, we have that
  \begin{align*}
    &\forall \tilde{h}\in \mathcal{U}_{\rho}(m,b,s),\forall \tilde{x}\in\mathcal{X}_{\rho}\quad &&\mathbb{E}_{c\sim\Pstar}[\tilde{h}(\tilde x, c)]&&\le \mathbb{E}_{c\sim\Phat_t}[\tilde{h}(\tilde x, c)]+\mathcal{R}_{\mathcal{\Pstar}}(\varepsilon,\tilde{h}(\tilde x,\cdot))\Rightarrow \\
    &\forall h\in \mathcal{U}(m,b,s),\forall x\in\mathcal{X}\quad &&\mathbb{E}_{c\sim\Pstar}[h( x, c)]&&\le \mathbb{E}_{c\sim\Phat_t}[h( x, c)]+\mathcal{R}_{\mathcal{\Pstar}}(\varepsilon,h(x,\cdot))+\min_{\tilde{h}\in \mathcal{U}_{\rho,t}}\|h-\tilde{h}\|_{\infty}+L_x\min_{\tilde{x}\in\mathcal{X}_{\rho}}\|x-\tilde{x}\|\\
    & && &&\le \mathbb{E}_{c\sim\Phat_t}[h( x, c)]+\mathcal{R}_{\mathcal{\Pstar}}(\varepsilon,h(x,\cdot))+L_x\rho.
  \end{align*}
  Hence,
  \begin{align*}
    &\mathbb{P}(\forall h\in\mathcal{U}(m,b,s),\forall x\in\mathcal{X}:\quad   \mathbb{E}_{c \sim \Pstar}[h(x, c)] \le  \mathbb{E}_{c \sim \Phat_t}[h(x, c)]  + \mathcal{R}_{\mathcal{\Pstar}}(\varepsilon,h(x,\cdot)) + (1+L_x)\rho)\\
    &\ge 1-\sum_{\tilde{h}\in \mathcal{U}_{\rho,t}}\sum_{\tilde{x}\in\mathcal{X}_{\rho}}\mathbb{P}(\mathbb{E}_{c \sim \Pstar}[\tilde{h}(\tilde x, c)] \ge\mathbb{E}_{c \sim \Phat_t}[\tilde{h}(\tilde x, c)]  + \mathcal{R}_{\mathcal{\Pstar}}(\varepsilon,\tilde h(\tilde x,\cdot)))\\
    &\ge 1-|\mathcal{X}_{\rho}||\mathcal{U}_{\rho}(m,b,s)|e^{-\varepsilon^2 \frac{t}{2\diamC^2}},
  \end{align*}
  where the last inequality follows from Theorem~1 in~\cite{gao2023finite}. The result follows from picking $\varepsilon=\varepsilon_t(\rho)$.
\end{proof}

\begin{lem}\label{lem:data_driven_wdro}
Let $0<\delta<1$ be a failure probability, and let $h \in \mathcal{U}(m,b,s)$. Let 
\[\varepsilon_t(\rho)=\sqrt{2\diamC^2\frac{ \log 1/\delta+ d_x\log(1+2\rho^{-1}\diamX)+\log \mathcal{N}_{\infty}(\rho,\mathcal{U}(m,b,s))}{t}}.
\] 
Then, with probability at least $1-\delta$, and for any $x\in\mathcal{X}$, we have that
\begin{equation}
    | \mathbb{E}_{c \sim \Pstar}[h(x, c)] - \mathbb{E}_{c \sim \Phat_t}[h(x, c)] | \leq \varepsilon_t(\rho) L^h_c(x) + (1+L_x)\rho,
\end{equation}
where $L^h_c(x)$ denotes the Lipschitz constant of function $h(x,\cdot)$ with respect to the context $c$, and $L_x$ is such that $\forall h \in \mathcal{U}(m,b,s), x, \tilde{x} \in \mathcal{X}, c\in\mathcal{C}$, we have $|h(x,c)-h(\tilde{x},c)|\leq L_x \|x-\tilde{x}\|$.
\end{lem}

\begin{proof}
From Lemma~\ref{lem:adapted_gao}, with probability at least $1-\delta$, for all $h \in \mathcal{U}(m,b,s)$ and all $x \in \mathcal{X}$, we have:
\begin{equation}
    \mathbb{E}_{c \sim \Pstar}[h(x, c)] \leq \mathbb{E}_{c \sim \Phat_t}[h(x, c)] + \varepsilon_t(\rho) L^h_c(x) + (1+L_x)\rho.
\end{equation}

Applying the same result to the function $-h(x,c)$, which also belongs to $\mathcal{U}(m,b,s)$ with the same Lipschitz constants, we get:
\begin{equation}
    \mathbb{E}_{c \sim \Pstar}[-h(x, c)] \leq \mathbb{E}_{c \sim \Phat_t}[-h(x, c)] + \varepsilon_t(\rho) L^{-h}_c(x) + (1+L_x)\rho.
\end{equation}

Since $L^{-h}_c(x) = L^h_c(x)$, this can be rewritten as:
\begin{equation}
    -\mathbb{E}_{c \sim \Pstar}[h(x, c)] \leq -\mathbb{E}_{c \sim \Phat_t}[h(x, c)] + \varepsilon_t(\rho) L^h_c(x) + (1+L_x)\rho.
\end{equation}

Rearranging terms:
\begin{equation}
    \mathbb{E}_{c \sim \Phat_t}[h(x, c)] - \varepsilon_t(\rho) L^h_c(x) - (1+L_x)\rho \leq \mathbb{E}_{c \sim \Pstar}[h(x, c)].
\end{equation}

Combining both inequalities, we obtain:
\begin{equation}
    | \mathbb{E}_{c \sim \Pstar}[h(x, c)] - \mathbb{E}_{c \sim \Phat_t}[h(x, c)] | \leq \varepsilon_t(\rho) L^h_c(x) + (1+L_x)\rho,
\end{equation}
which holds with probability at least $1-\delta$ for all $h \in \mathcal{U}(m,b,s)$ and all $x \in \mathcal{X}$.
\end{proof}

By choosing $\rho=t^{-1}$ we can derive the following:
\begin{lem}\label{cor:data_driven_UCB}
Let $0<\delta<1$ be a failure probability, and let $\UCB(x,c) = \mu_{t-1}(x,c) + \beta_t \sigma_{t-1}(x,c) \in \mathcal{U}(\bar{B}_t,\bar{B}_t,\lambda)$. Then, with probability at least $1-2\delta$, for any $x\in\mathcal{X}$ we have:
\begin{equation}
    | \mathbb{E}_{c \sim \Pstar}[\UCB(x, c)] - \mathbb{E}_{c \sim \Phat_t}[\UCB(x, c)] | \leq \varepsilon_t L^\UCB(x) + \rho_t,
\end{equation}
where 
\[\varepsilon_t=\sqrt{2\diamC^2\frac{\log 1/\delta+ d_x\log(1+2t\diamX)+\log \mathcal{N}_{\infty}(t^{-1},\mathcal{U}(\bar{B}_t,\bar{B}_t,\lambda))}{t}},\]
and \[\rho_t = (1+2L\bar{B}_t)t^{-1},\]
with $\bar{B}_t$ as defined in Lemma~\ref{lem:B_bar}.
\end{lem}
\begin{proof}
    This follows from the fact that the UCB function belongs to the class $\mathcal{U}(\bar{B}_t,\bar{B}_t,\lambda)$, as per Lemma~\ref{lem:B_bar}, with probability $1-\delta$, we have $\|\mu_{t-1}\|_{\mathcal{H}_k} \leq \bar{B}_t$ and $\beta_t \leq \bar{B}_t$. By definition of $\sigma_{t-1}(z) = \|(\lambda I + V_{t-1})^{-1/2} k(\cdot, z)\|_{\mathcal{H}_k}$ we know that $\|(\lambda I + V_{t-1})^{-1/2}\|_{op} \leq \lambda^{-1/2}$.
\end{proof}

We can now derive the instantaneous and cumulative expected regret for the data-driven setting.
\begin{thm}[Data-driven instantaneous expected regret]\label{thm:inst_regret_data_appendix}
    Let Assumption~\ref{ass:Lip} hold. Fix a failure probability $0<\delta<1$.
    With probability at least $1-2\delta$, for all $t\ge 1$ the instantaneous expected regret for the data-driven setting can be bounded by
    \begin{equation}
        r_t \leq \EPstar \left[ 2 \beta_t \sigma_{t-1}(\xt, c) \right] + 2 \varepsilon_t L^{\UCB}(\xstar_t) + 2\rho_t
    \end{equation}
    where $\varepsilon_t$ and $\rho_t$ are as defined in Corollary~\ref{cor:data_driven_UCB}.
\end{thm}

\begin{proof}
The proof follows the same structure as Theorem~\ref{thm:inst_regret_appendix}, but using Corollary~\ref{cor:data_driven_UCB} instead of Lemma~\ref{cor:wdro}. 
% We start from the definition of instantaneous regret:
%     \begin{equation}
%         \begin{aligned}
%             r_t &= \EPstar{[\fxtstar]} - \EPstar{[\fxt]}\\
%             %
%             &\leq \EPstar\left[\UCB(\xstar_t,c)\right] -\EPstar\left[\LCB(\xt,c)\right]\\
%             %
%             &\leq \EPhatt\left[\UCB(\xstar_t,c) \right] + \varepsilon_t(t^{-1}) L^{\UCB}(\xstar_t) + (1+L^{\UCB})t^{-1} -\EPstar\left[\LCB(\xt,c)\right]\\
%             %
%             &\leq \EPhatt\left[\UCB(\xt,c) \right] - \varepsilon_t(t^{-1}) L^{\UCB}(\xt) + 2\varepsilon_t(t^{-1}) L^{\UCB}(\xstar_t) + 2(1+L^{\UCB})t^{-1} -\EPstar\left[\LCB(\xt,c)\right] \\
%             %
%             &\leq \EPstar\left[\UCB(\xt,c) \right] + 2\varepsilon_t(t^{-1}) L^{\UCB}(\xstar_t) + 2(1+L^{\UCB})t^{-1} -\EPstar\left[\LCB(\xt,c)\right] \\
%             %
%             &\leq \EPstar [\mu_{t-1}(\xt,c) + \beta_t \sigma_{t-1}(\xt, c) - \mu_{t-1}(\xt,c) + \beta_t \sigma_{t-1}(\xt, c)] + 2\varepsilon_t(t^{-1}) L^{\UCB}(\xstar_t) + 2(1+L^{\UCB})t^{-1} \\
%             %
%             &= \EPstar \left[ 2 \beta_t \sigma_{t-1}(\xt, c) \right] + 2\varepsilon_t(t^{-1}) L^{\UCB}(\xstar_t) + 2(1+L^{\UCB})t^{-1}
%         \end{aligned}
%     \end{equation}
\end{proof}
\begin{thm}[Data-driven cumulative expected regret]\label{thm:cum_regret_data_corrected}
    Let Assumption~\ref{ass:Lip} hold and let $L^{\UCB}(x)$ be a Lipschitz constant with respect to the context $c$ for $\UCB(x,c)$. Fix a failure probability $0<\delta<1$.
    With probability at least $1-3\delta$, the cumulative expected regret for the data-driven setting can be bounded as:
    \begin{equation}
        R_T \leq 4 \beta_T \sqrt{T(\gamma_T + 4 \log(6/\delta))} + \sum_{t=1}^{T} \left(2\varepsilon_t L^{\UCB}(\xstar_t) + 2\rho_t \right),
    \end{equation}
    where $\gamma_{T}$ is the maximum information gain at time $T$, $\varepsilon_t$ and $\rho_t$ are as defined in Corollary~\ref{cor:data_driven_UCB}.
\end{thm}

\begin{proof}
    The proof follows from the same steps as Theorem~\ref{thm:cum_regret_appendix}.
    % Starting from the definition of cumulative expected regret and applying Theorem~\ref{thm:inst_regret_data_corrected}:
    % \begin{equation}
    %     \begin{aligned}
    %         R_T &= \sum_{t=1}^{T} r_t = \sum_{t=1}^{T} \EPstar{[\fxtstar]} - \EPstar{[\fxt]} \\
    %         &\leq \sum_{t=1}^{T} \EPstar \left[ 2 \beta_t \sigma_{t-1}(\xt, c) \right] + \sum_{t=1}^{T} \left(2\varepsilon_t(t^{-1}) L^{\UCB}(\xstar_t) + 2(1+L^{\UCB})t^{-1}\right) \\
    %         &\leq 2 \beta_T \sum_{t=1}^{T} \EPstar \left[\sigma_{t-1}(\xt, c) \right] + \sum_{t=1}^{T} \left(2\varepsilon_t(t^{-1}) L^{\UCB}(\xstar_t) + 2(1+L^{\UCB})t^{-1}\right) \\
    %     \end{aligned}
    % \end{equation}
    
    % The first term can be bounded as in the proof of Theorem~\ref{thm:cum_regret}:
    % \begin{equation}
    %     \begin{aligned}
    %         2 \beta_T \sum_{t=1}^{T} \EPstar \left[\sigma_{t-1}(\xt, c) \right] &\leq 2 \beta_T \sqrt{T \sum_{t=1}^{T} \left(\EPstar \left[ \sigma_{t-1}(\xt, c) \right]\right)^2} \\
    %         &\leq 2 \beta_T \sqrt{T \sum_{t=1}^{T} \EPstar \left[ \sigma_{t-1}(\xt, c)^2 \right]} \\
    %         &\leq 4 \beta_T \sqrt{T(\gamma_T + 4 \log(6/\delta))}
    %     \end{aligned}
    % \end{equation}
    
    % Combining these results yields the stated bound.
\end{proof}

Using a result from~\cite{yang2020function} we obtain a bound on $\mathcal{N}_{\infty}(\rho,\mathcal{U}(m,b,s))$, specialized here to the squared exponential kernel. The proposed analysis works in more general settings with kernels experiencing either exponential or polynomial eigendecay, see e.g.~\cite{yang2020function, vakili2023kernelized}.

\begin{lem}[Adapted from Lemma D.1 of~\cite{yang2020function}]\label{app_lem:covering_square_exponential}
    Let $z\in\mathcal{Z}\subset\mathbb{R}^d$, let $k(z,z')$ be the squared exponential kernel with $k(z,z)\leq 1$, and let $0<s<1$ Then, there exist a global constant $\kappa$ such that
    $$\log \mathcal{N}_{\infty}(\rho,\mathcal{U}(m,b,s)) \leq \kappa \left(1+\log\frac{m}{\rho}\right)^{1+d} + \kappa \left(1+\log\frac{1}{\rho}\right)^{1+2d}$$
\end{lem}

\begin{cor}[Data-driven WDRBO Regret Order for Squared Exponential Kernel]\label{cor:datadriven_SE}
    Let $0< \delta < 1$ be a failure probability and let Assumption~\ref{ass:Lip} hold. For the Squared Exponential kernel, with probability $1-3\delta$, the cumulative expected regret in the data-driven setting is bounded by:
    \begin{equation}
        R_T = \tilde{O}\left(\sqrt{T}\right)
    \end{equation}
    where $\tilde{O}$ omits logarithmic terms.
\end{cor}

\begin{proof}
    For the Squared Exponential kernel, we have $\gamma_t = O(\log^{d+1}(t))$ from \cite{vakili2021information}. From Lemma~\ref{app_lem:covering_square_exponential}, we know:
    $$\log \mathcal{N}_{\infty}(t^{-1},\mathcal{U}(\bar{B}_t,\bar{B}_t,\lambda)) \leq \kappa \left(1+\log(t\bar{B}_t)\right)^{1+d} + \kappa \left(1+\log t\right)^{1+2d}$$
    
    From Lemma~\ref{lem:B_bar}, $\bar{B}_t = O(\sqrt{\gamma_t} + \sqrt{\log(1/\delta)}) = \tilde{O}(1)$ for the Squared Exponential kernel.
    
    Therefore:
    $$\log \mathcal{N}_{\infty}(t^{-1},\mathcal{U}(\bar{B}_t,\bar{B}_t,\lambda)) = \tilde{O}(1)$$
    
    The radius $\varepsilon_t$ becomes:
    \begin{align}
        \varepsilon_t &= \sqrt{2\diamC^2\frac{\log(1/\delta) + d_x\log(1+2t\diamX) + \log \mathcal{N}_{\infty}(t^{-1},\mathcal{U}(\bar{B}_t,\bar{B}_t,\lambda))}{t}} \\
        &= \tilde{O}\left(\frac{1}{\sqrt{t}}\right)
    \end{align}
    
    Substituting these results into Theorem~\ref{thm:cum_regret_data}, we get:
    \begin{align}
        R_T &\leq 4\beta_T\sqrt{T(\gamma_T + 4\log(6/\delta))} + \sum_{t=1}^{T}\left(2\varepsilon_tL^{\UCB}(\xstar_t) + 2\rho_t\right) \\
        &= \tilde{O}(\sqrt{T\gamma_T}) + \tilde{O}\left(\sum_{t=1}^{T}\frac{L^{\UCB}(\xstar_t)}{\sqrt{t}}\right) + \tilde{O}\left(\sum_{t=1}^{T}\frac{(1+2L\bar{B}_t)}{t}\right) \\
        &= \tilde{O}(\sqrt{T\log^{d+1}(T)}) + \tilde{O}\left(\bar{B}_T L \sum_{t=1}^{T}\frac{1}{\sqrt{t}}\right) + \tilde{O}\left(\sum_{t=1}^{T}\frac{1}{t}\right) \\
        &= \tilde{O}(\sqrt{T}) + \tilde{O}(\sqrt{T}) + O(\log T) \\
        &= \tilde{O}(\sqrt{T})
    \end{align}
    
    Where we used the fact that $L^{\UCB}(\xstar_t) \leq 2\bar{B}_t L = \tilde{O}(1)$ from Lemma~\ref{lem:B_bar}, and that $\sum_{t=1}^{T}\frac{1}{\sqrt{t}} = O(\sqrt{T})$ and $\sum_{t=1}^{T}\frac{1}{t} = O(\log T)$.
\end{proof}

\section{EXPERIMENTS DETAILS AND ADDITIONAL EXPERIMENTS}
The experimental setup is based on an adaptation of the work of~\cite{huang2024stochastic}. We exploit their implementation of the methods GP-UCB, SBO-KDE, DRBO-KDE, DRBO-MMD, DRBO-MMD Minmax, and StableOpt, available at \href{https://github.com/lamda-bbo/sbokde}{https://github.com/lamda-bbo/sbokde}, with only minor changes to make the code compatible with our hardware. We refer to Appendix B.1 of~\cite{huang2024stochastic} for more details. The algorithms, including ERBO and WDRBO, are implemented in BoTorch~\cite{balandat2020botorch}, with the inner convex optimization problems of DRBO-KDE and DRBO-MMD solved using CVXPY~\cite{diamond2016cvxpy}. Our code is available at the following link~\href{https://github.com/frmicheli/WDRBO}{https://github.com/frmicheli/WDRBO} .

We also exploited the implementations of the functions Ackley, Hartmann, Modified Branin, Newsvendor, Portfolio (Normal), and Portfolio (Uniform) of~\cite{huang2024stochastic}. We refer to Appendix B.2 of~\cite{huang2024stochastic} for more details. The only variation has been in the choice of context distribution for Ackley, Hartmann, Modified Branin, Portfolio (Normal) where we used $c\sim\mathcal{N}(0.5,0.2^2)$ with $c$ clipped to $[0,1]$.
We implemented the Three Humps Camel function that is a standard benchmark function for global optimization algorithms. The input space is two-dimensional, we restricted it to the domain $x \in [-1, 1]$ and $c \in [-1, 1]$, and chose a uniform distribution for the context $c$.
We function is defined as follows: 
\begin{equation}
f(x) = 2x^2 - 1.05x^4 + \frac{x^6}{6} + xc + c^2\ .
\end{equation}

In Fig.~\ref{fig:2} we compare the performance of all the algorithms including DRBO-MMD, DRBO-MMD Minmax, and StableOpt, which we did not show in the main text. In Fig.~\ref{fig:3} we show the instantaneous regret for all the algorithms which can help better compare the asymptotic performance of the different algorithms.

Table~\ref{table:2} reports the mean and standard error of computational times in seconds for the Ackley and Branin functions for all the considered algorithms.
\begin{table}[ht]
\centering
\begin{tabular}{|l|c|c|}
\hline
$ $ &\textbf{Ackley} & \textbf{Branin} \\ \hline
WDRBO (ours) & $44.3\pm 2.2$ & $54.5 \pm 2.6$ \\ \hline
ERBO (ours) & $43.8\pm 1.6$& $45.2 \pm 2.5$ \\ \hline
GP-UCB & $15.7\pm 1.4$ & $15.1 \pm 1.0$\\ \hline
SBO-KDE & $46.7\pm 0.7$ & $49.7\pm 1.5$ \\ \hline
DRBO-KDE & $599.7 \pm 33.0$ & $525.0 \pm 71.7$ \\ \hline
DRBO-MMD & $644.9 \pm 45.6$ & $131.6 \pm 9.6$ \\ \hline
DRBO-MMD Minimax & $104.8 \pm 3.9$ & $21.1 \pm 1.8$ \\ \hline
StableOpt & $77.6 \pm 2.4$ & $64.7 \pm 4.2$ \\ \hline
\end{tabular}
\caption{Mean and standard error of computational times in seconds for the Ackley ($d_x = 1$, $d_c = 1$) and Branin ($d_x = 2$, $d_c = 2$) functions.}\label{table:2}
\end{table}

\begin{figure*}[h]
  \includegraphics[width=\textwidth]{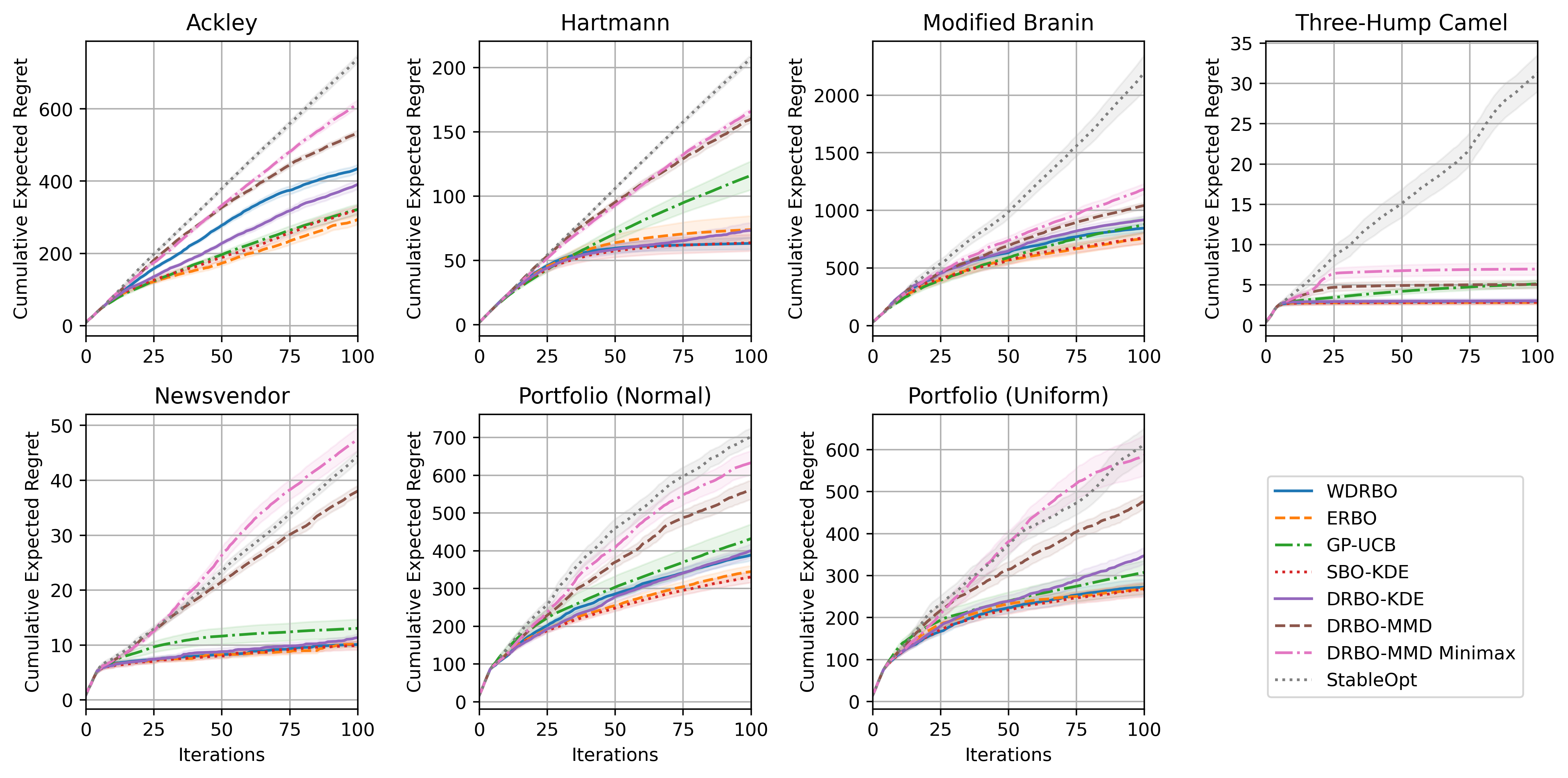}
  \caption{Mean and standard error of the cumulative expected regret. }
  \label{fig:2}
\end{figure*}

\begin{figure*}[h]
  \includegraphics[width=\textwidth]{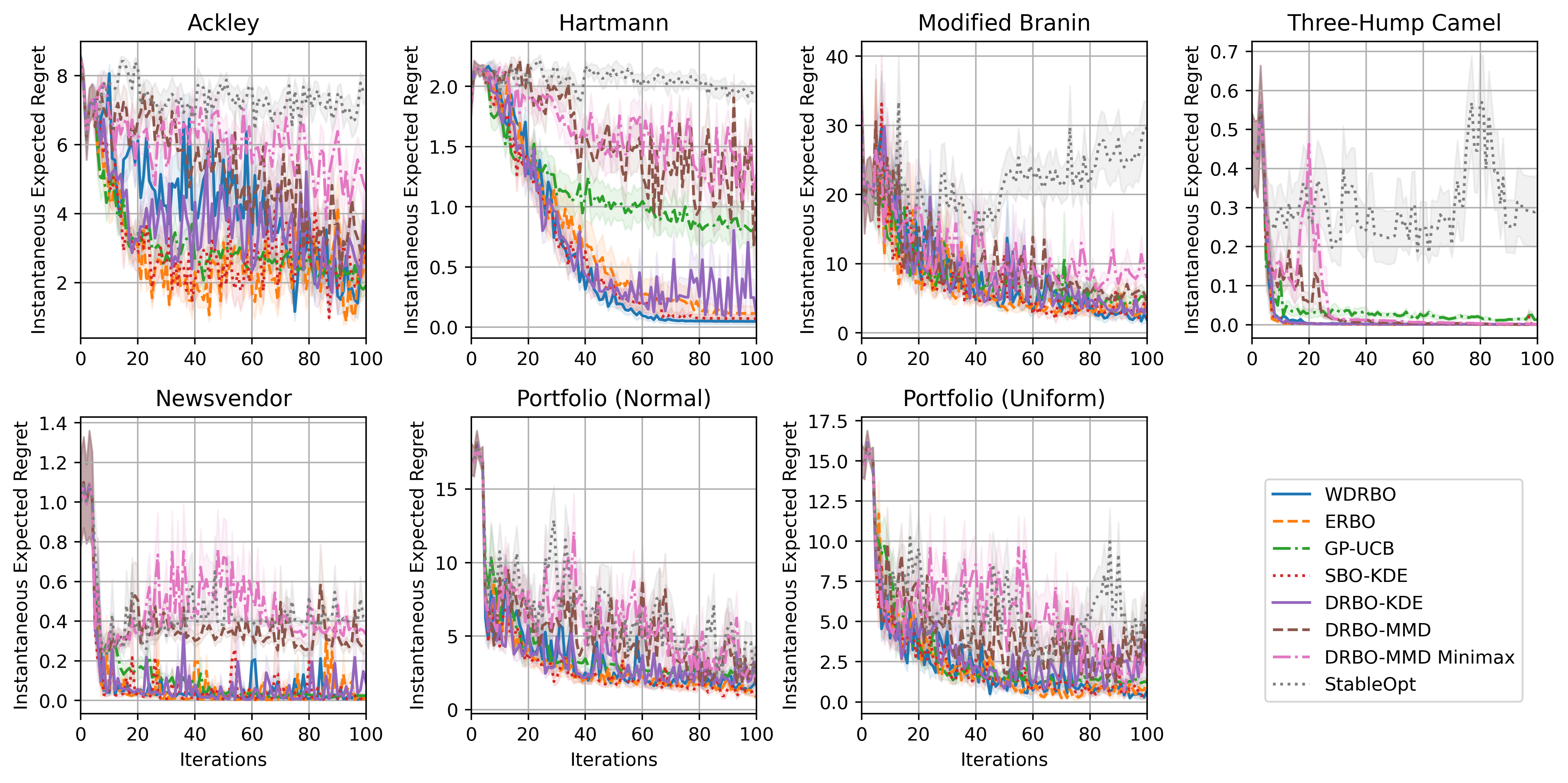}
  \caption{Mean and standard error of the instantaneous expected regret. }
  \label{fig:3}
\end{figure*}

%% file: main.bbl
\begin{thebibliography}{37}
\providecommand{\natexlab}[1]{#1}
\providecommand{\url}[1]{\texttt{#1}}
\expandafter\ifx\csname urlstyle\endcsname\relax
  \providecommand{\doi}[1]{doi: #1}\else
  \providecommand{\doi}{doi: \begingroup \urlstyle{rm}\Url}\fi

\bibitem[Ueno et~al.(2016)Ueno, Rhone, Hou, Mizoguchi, and Tsuda]{ueno2016combo}
Tsuyoshi Ueno, Trevor~David Rhone, Zhufeng Hou, Teruyasu Mizoguchi, and Koji Tsuda.
\newblock Combo: An efficient bayesian optimization library for materials science.
\newblock \emph{Materials discovery}, 4:\penalty0 18--21, 2016.

\bibitem[Li et~al.(2019)Li, Rana, Gupta, Nguyen, Venkatesh, Sutti, Rubin, Slezak, Height, Mohammed, et~al.]{li2019accelerating}
Cheng Li, Santu Rana, Sunil Gupta, Vu~Nguyen, Svetha Venkatesh, Alessandra Sutti, David Rubin, Teo Slezak, Murray Height, Mazher Mohammed, et~al.
\newblock Accelerating experimental design by incorporating experimenter hunches.
\newblock \emph{arXiv preprint arXiv:1907.09065}, 2019.

\bibitem[Ru et~al.(2020)Ru, Alvi, Nguyen, Osborne, and Roberts]{ru2020bayesian}
Binxin Ru, Ahsan Alvi, Vu~Nguyen, Michael~A Osborne, and Stephen Roberts.
\newblock Bayesian optimisation over multiple continuous and categorical inputs.
\newblock In \emph{International Conference on Machine Learning}, pages 8276--8285. PMLR, 2020.

\bibitem[Shahriari et~al.(2015)Shahriari, Swersky, Wang, Adams, and De~Freitas]{shahriari2015taking}
Bobak Shahriari, Kevin Swersky, Ziyu Wang, Ryan~P Adams, and Nando De~Freitas.
\newblock Taking the human out of the loop: A review of bayesian optimization.
\newblock \emph{Proceedings of the IEEE}, 104\penalty0 (1):\penalty0 148--175, 2015.

\bibitem[Krause and Ong(2011)]{krause2011contextual}
Andreas Krause and Cheng Ong.
\newblock Contextual gaussian process bandit optimization.
\newblock \emph{Advances in neural information processing systems}, 24, 2011.

\bibitem[Valko et~al.(2013)Valko, Korda, Munos, Flaounas, and Cristianini]{valko2013finite}
Michal Valko, Nathaniel Korda, R{\'e}mi Munos, Ilias Flaounas, and Nelo Cristianini.
\newblock Finite-time analysis of kernelised contextual bandits.
\newblock \emph{arXiv preprint arXiv:1309.6869}, 2013.

\bibitem[Kirschner and Krause(2019)]{kirschner2019stochastic}
Johannes Kirschner and Andreas Krause.
\newblock Stochastic bandits with context distributions.
\newblock \emph{Advances in Neural Information Processing Systems}, 32, 2019.

\bibitem[Rahimian and Mehrotra(2019)]{rahimian2019distributionally}
Hamed Rahimian and Sanjay Mehrotra.
\newblock Distributionally robust optimization: A review.
\newblock \emph{arXiv preprint arXiv:1908.05659}, 2019.

\bibitem[Kuhn et~al.(2019)Kuhn, Esfahani, Nguyen, and Shafieezadeh-Abadeh]{kuhn2019wasserstein}
Daniel Kuhn, Peyman~Mohajerin Esfahani, Viet~Anh Nguyen, and Soroosh Shafieezadeh-Abadeh.
\newblock Wasserstein distributionally robust optimization: Theory and applications in machine learning.
\newblock In \emph{Operations research \& management science in the age of analytics}, pages 130--166. Informs, 2019.

\bibitem[Gao et~al.(2024)Gao, Chen, and Kleywegt]{gao2024wasserstein}
Rui Gao, Xi~Chen, and Anton~J Kleywegt.
\newblock Wasserstein distributionally robust optimization and variation regularization.
\newblock \emph{Operations Research}, 72\penalty0 (3):\penalty0 1177--1191, 2024.

\bibitem[Kirschner et~al.(2020)Kirschner, Bogunovic, Jegelka, and Krause]{kirschner2020distributionally}
Johannes Kirschner, Ilija Bogunovic, Stefanie Jegelka, and Andreas Krause.
\newblock Distributionally robust bayesian optimization.
\newblock In \emph{International Conference on Artificial Intelligence and Statistics}, pages 2174--2184. PMLR, 2020.

\bibitem[Nguyen et~al.(2020)Nguyen, Gupta, Ha, Rana, and Venkatesh]{nguyen2020distributionally}
Thanh Nguyen, Sunil Gupta, Huong Ha, Santu Rana, and Svetha Venkatesh.
\newblock Distributionally robust bayesian quadrature optimization.
\newblock In \emph{International Conference on Artificial Intelligence and Statistics}, pages 1921--1931. PMLR, 2020.

\bibitem[Husain et~al.(2024)Husain, Nguyen, and van~den Hengel]{husain2024distributionally}
Hisham Husain, Vu~Nguyen, and Anton van~den Hengel.
\newblock Distributionally robust bayesian optimization with $\phi$-divergences.
\newblock \emph{Advances in Neural Information Processing Systems}, 36, 2024.

\bibitem[Tay et~al.(2022)Tay, Foo, Daisuke, Leong, and Low]{tay2022efficient}
Sebastian~Shenghong Tay, Chuan~Sheng Foo, Urano Daisuke, Richalynn Leong, and Bryan Kian~Hsiang Low.
\newblock Efficient distributionally robust bayesian optimization with worst-case sensitivity.
\newblock In \emph{International Conference on Machine Learning}, pages 21180--21204. PMLR, 2022.

\bibitem[Huang et~al.(2024)Huang, Song, Xue, and Qian]{huang2024stochastic}
Xiaobin Huang, Lei Song, Ke~Xue, and Chao Qian.
\newblock Stochastic bayesian optimization with unknown continuous context distribution via kernel density estimation.
\newblock In \emph{Proceedings of the AAAI Conference on Artificial Intelligence}, volume~38, pages 12635--12643, 2024.

\bibitem[Srinivas et~al.(2009)Srinivas, Krause, Kakade, and Seeger]{srinivas2009gaussian}
Niranjan Srinivas, Andreas Krause, Sham~M Kakade, and Matthias Seeger.
\newblock Gaussian process optimization in the bandit setting: No regret and experimental design.
\newblock \emph{arXiv preprint arXiv:0912.3995}, 2009.

\bibitem[Srinivas et~al.(2012)Srinivas, Krause, Kakade, and Seeger]{srinivas2012information}
Niranjan Srinivas, Andreas Krause, Sham~M Kakade, and Matthias~W Seeger.
\newblock Information-theoretic regret bounds for gaussian process optimization in the bandit setting.
\newblock \emph{IEEE transactions on information theory}, 58\penalty0 (5):\penalty0 3250--3265, 2012.

\bibitem[Abbasi-Yadkori(2013)]{abbasi2013online}
Yasin Abbasi-Yadkori.
\newblock Online learning for linearly parametrized control problems.
\newblock 2013.

\bibitem[Whitehouse et~al.(2023)Whitehouse, Wu, and Ramdas]{whitehouse2023improved}
Justin Whitehouse, Zhiwei~Steven Wu, and Aaditya Ramdas.
\newblock Improved self-normalized concentration in hilbert spaces: Sublinear regret for gp-ucb.
\newblock \emph{arXiv preprint arXiv:2307.07539}, 2023.

\bibitem[Gao(2023)]{gao2023finite}
Rui Gao.
\newblock Finite-sample guarantees for wasserstein distributionally robust optimization: Breaking the curse of dimensionality.
\newblock \emph{Operations Research}, 71\penalty0 (6):\penalty0 2291--2306, 2023.

\bibitem[Kirschner and Krause(2018)]{kirschner2018information}
Johannes Kirschner and Andreas Krause.
\newblock Information directed sampling and bandits with heteroscedastic noise.
\newblock In \emph{Conference On Learning Theory}, pages 358--384. PMLR, 2018.

\bibitem[Bogunovic and Krause(2021)]{bogunovic2021misspecified}
Ilija Bogunovic and Andreas Krause.
\newblock Misspecified gaussian process bandit optimization.
\newblock \emph{Advances in neural information processing systems}, 34:\penalty0 3004--3015, 2021.

\bibitem[Sch{\"o}lkopf and Smola(2002)]{scholkopf2002learning}
Bernhard Sch{\"o}lkopf and Alexander~J Smola.
\newblock \emph{Learning with kernels: support vector machines, regularization, optimization, and beyond}.
\newblock MIT press, 2002.

\bibitem[Williams and Rasmussen(2006)]{williams2006gaussian}
Christopher~KI Williams and Carl~Edward Rasmussen.
\newblock \emph{Gaussian processes for machine learning}, volume~2.
\newblock MIT press Cambridge, MA, 2006.

\bibitem[Chowdhury and Gopalan(2017)]{chowdhury2017kernelized}
Sayak~Ray Chowdhury and Aditya Gopalan.
\newblock On kernelized multi-armed bandits.
\newblock In \emph{International Conference on Machine Learning}, pages 844--853. PMLR, 2017.

\bibitem[Vakili et~al.(2021)Vakili, Khezeli, and Picheny]{vakili2021information}
Sattar Vakili, Kia Khezeli, and Victor Picheny.
\newblock On information gain and regret bounds in gaussian process bandits.
\newblock In \emph{International Conference on Artificial Intelligence and Statistics}, pages 82--90. PMLR, 2021.

\bibitem[Auer(2002)]{auer2002finite}
P~Auer.
\newblock Finite-time analysis of the multiarmed bandit problem, 2002.

\bibitem[Fournier and Guillin(2015)]{fournier2015rate}
Nicolas Fournier and Arnaud Guillin.
\newblock On the rate of convergence in wasserstein distance of the empirical measure.
\newblock \emph{Probability theory and related fields}, 162\penalty0 (3):\penalty0 707--738, 2015.

\bibitem[Fournier(2022)]{fournier2022convergence}
Nicolas Fournier.
\newblock Convergence of the empirical measure in expected wasserstein distance: non asymptotic explicit bounds in $\mathbb{R}^d$.
\newblock \emph{arXiv preprint arXiv:2209.00923}, 2022.

\bibitem[Yang et~al.(2020)Yang, Jin, Wang, Wang, and Jordan]{yang2020function}
Zhuoran Yang, Chi Jin, Zhaoran Wang, Mengdi Wang, and Michael~I Jordan.
\newblock On function approximation in reinforcement learning: Optimism in the face of large state spaces.
\newblock \emph{arXiv preprint arXiv:2011.04622}, 2020.

\bibitem[Bogunovic et~al.(2018)Bogunovic, Scarlett, Jegelka, and Cevher]{bogunovic2018adversarially}
Ilija Bogunovic, Jonathan Scarlett, Stefanie Jegelka, and Volkan Cevher.
\newblock Adversarially robust optimization with gaussian processes.
\newblock \emph{Advances in neural information processing systems}, 31, 2018.

\bibitem[Kanagawa et~al.(2018)Kanagawa, Hennig, Sejdinovic, and Sriperumbudur]{kanagawa2018gaussian}
Motonobu Kanagawa, Philipp Hennig, Dino Sejdinovic, and Bharath~K Sriperumbudur.
\newblock Gaussian processes and kernel methods: A review on connections and equivalences.
\newblock \emph{arXiv preprint arXiv:1807.02582}, 2018.

\bibitem[Van~Waarde and Sepulchre(2022)]{van2022kernel}
Henk~J Van~Waarde and Rodolphe Sepulchre.
\newblock Kernel-based models for system analysis.
\newblock \emph{IEEE Transactions on Automatic Control}, 68\penalty0 (9):\penalty0 5317--5332, 2022.

\bibitem[Vershynin(2018)]{vershynin2018high}
Roman Vershynin.
\newblock \emph{High-dimensional probability: An introduction with applications in data science}, volume~47.
\newblock Cambridge university press, 2018.

\bibitem[Vakili and Olkhovskaya(2023)]{vakili2023kernelized}
Sattar Vakili and Julia Olkhovskaya.
\newblock Kernelized reinforcement learning with order optimal regret bounds.
\newblock \emph{Advances in Neural Information Processing Systems}, 36:\penalty0 4225--4247, 2023.

\bibitem[Balandat et~al.(2020)Balandat, Karrer, Jiang, Daulton, Letham, Wilson, and Bakshy]{balandat2020botorch}
Maximilian Balandat, Brian Karrer, Daniel Jiang, Samuel Daulton, Ben Letham, Andrew~G Wilson, and Eytan Bakshy.
\newblock Botorch: A framework for efficient monte-carlo bayesian optimization.
\newblock \emph{Advances in neural information processing systems}, 33:\penalty0 21524--21538, 2020.

\bibitem[Diamond and Boyd(2016)]{diamond2016cvxpy}
Steven Diamond and Stephen Boyd.
\newblock Cvxpy: A python-embedded modeling language for convex optimization.
\newblock \emph{Journal of Machine Learning Research}, 17\penalty0 (83):\penalty0 1--5, 2016.

\end{thebibliography}
